\documentclass[webpdf,imanum]{ima-authoring-template}%
\usepackage{soul}
\graphicspath{{Fig/}}

\theoremstyle{thmstyletwo}%
\newtheorem{theorem}{Theorem}[section]
\newtheorem{proposition}[theorem]{Proposition}%
\newtheorem{remark}[theorem]{Remark}%

\numberwithin{equation}{section}

\newtheorem{lemma}[theorem]{Lemma}%
\newtheorem{corollary}[theorem]{Corollary}%
\newtheorem{assumption}[theorem]{Assumption}

\usepackage{amssymb}

\usepackage{booktabs}

\usepackage{bm}
\renewcommand{\vec}{\mathbf}

\newcommand{\R}{\mathbb{R}}
\newcommand{\E}{\mathbb{E}}

\newcommand{\T}{\mathsf{T}}
\newcommand{\F}{\mathsf{F}}

\newcommand{\ATA}{\vec{A}^\T\vec{A}}

\newcommand{\lmin}{\lambda_{\textup{min}}}
\newcommand{\lmax}{\lambda_{\textup{max}}}
\newcommand{\lave}{\lambda_{\textup{ave}}}

\newcommand{\condCbd}{M} 

\usepackage[capitalize,nameinlink,noabbrev]{cleveref}
\crefformat{equation}{#2(#1)#3}
\crefmultiformat{equation}{#2(#1)#3}%
{ and~#2(#1)#3}{, #2(#1)#3}{ and~#2(#1)#3}

\begin{document}

\copyrightyear{2024}
\firstpage{1}


\title[On the fast convergence of minibatch heavy ball momentum]{On the fast convergence of minibatch heavy ball momentum}

\author{Raghu Bollapragada
\address{\orgdiv{Operations Research and Industrial Engineering}, \orgname{The University of Texas at Austin}, \orgaddress{\street{204 E. Dean Keeton Street}, \postcode{78712}, \state{TX}, \country{USA}}}}
\author{Tyler Chen
\address{\orgdiv{Mathematics}, \orgname{New York University}, \orgaddress{\street{251 Mercer Street}, \postcode{10012}, \state{NY}, \country{USA}}\\
\orgdiv{Computer Science and Engineering}, \orgname{New York University}, \orgaddress{\street{370 Jay Street}, \postcode{11201}, \state{NY}, \country{USA}}}}
\author{Rachel Ward
\address{\orgdiv{Mathematics}, \orgname{The University of Texas at Austin}, \orgaddress{\street{2515 Speedway}, \postcode{78712}, \state{TX}, \country{USA}}\\
\orgdiv{Computational Engineering and Sciences}, \orgname{The University of Texas at Austin}, \orgaddress{\street{201 E. 24th Street}, \postcode{78712}, \state{TX}, \country{USA}}}}

\authormark{Bollapragada \emph{et al.}}


\received{Date}{0}{Year}
\revised{Date}{0}{Year}
\accepted{Date}{0}{Year}

\editor{Associate Editor: Name}

\abstract{
Simple stochastic momentum methods are widely used in machine learning optimization, but their good practical performance is at odds with an absence of theoretical guarantees of acceleration in the literature.
In this work, we aim to close the gap between theory and practice by showing that stochastic heavy ball momentum retains the fast linear rate of (deterministic) heavy ball momentum on quadratic optimization problems, at least when minibatching with a sufficiently large batch size. 
The algorithm we study can be interpreted as an accelerated randomized Kaczmarz algorithm with minibatching and heavy ball momentum.
The analysis relies on carefully decomposing the momentum transition matrix, and using new spectral norm concentration bounds for products of independent random matrices.
We provide numerical illustrations demonstrating that our bounds are reasonably sharp.}


\maketitle

\section{Introduction}

The success of learning algorithms trained with stochastic gradient descent (SGD) variants, dating back to the seminal AlexNet architecture \citep{krizhevsky2012imagenet}—arguably initiating the ``deep learning revolution"—and empirically verified comprehensively in \citep{sutskever2013importance}, emphasizes the importance of incorporating simple momentum for achieving rapid convergence in neural network learning. 
Although more complex momentum (or acceleration) algorithms have been proposed \citep{nesterov83,jin2018accelerated, bubeck2015geometric,liu2015accelerated,jain2018accelerating,van2017fastest,cyrus2018robust,allen2018katyusha}, demonstrating faster convergence rates than plain SGD on standard classes of loss functions, Polyak's original simple momentum update \citep{polyak1964methods} defies theory by remaining highly effective in practice and remains a popular choice for many applications. 
Despite several studies analyzing the performance of stochastic momentum methods \citep{kidambi2018insufficiency, can2019accelerated,loizou2017linearly, loizou2020momentum, gitman2019understanding, liu2020improved, flammarion2015averaging, gadat2018, sebbouh2021almost, yan2018unified}, a gap persists between existing theoretical guarantees and their superior practical performance. 
We aim to bridge this gap by analyzing the properties of simple stochastic momentum methods in the context of quadratic optimization.

Given a $n\times d$ matrix $\vec{A}$ and a length $n$ vector $\vec{b}$, the linear least squares problem
\begin{align}
\label{eqn:lls}
 \min_{x \in \R^d}\; \; & f(\vec{x}); \quad \quad f(\vec{x}) := \frac{1}{2} \| \vec{A} \vec{x} - \vec{b} \|^2 =\frac{1}{2}  \sum_{i=1}^{n}| \vec{a}_i^\T \vec{x} - b_i |^2
= \sum_{i=1}^{n} f_i(\vec{x})
\end{align}
is one of the most fundamental problems in optimization.
One approach to solving  \cref{eqn:lls} is Polyak's heavy ball momentum (HBM) \citep{polyak1964methods}, also called `standard' or `classical' momentum.
\ref{heavyball} updates the parameter estimate $\vec{x}_k$ for the solution $\vec{x}^*$     as
\begin{equation*}
    \vec{x}_{k+1} = \vec{x}_k - \alpha_k \nabla f(\vec{x}_k) + \beta_k \vec{m}_k
    ,\qquad
    \vec{m}_{k+1} = \beta_k \vec{m}_k - \alpha_k \nabla f(\vec{x}_k),
\end{equation*}
where $\alpha_k$ and  $\beta_k$ are the step-size and momentum parameters respectively.
This is equivalent to the update
\begin{equation}
\label{heavyball}
\tag{HBM}
    \vec{x}_{k+1} = \vec{x}_k - \alpha_k \nabla f(\vec{x}_k) + \beta_k (\vec{x}_k - \vec{x}_{k-1}).
\end{equation}

The gradient of the objective \cref{eqn:lls} is easily computed to be $\nabla f(\vec{x}) = \vec{A}^\T(\vec{A}\vec{x}- \vec{b})$, and when $\ATA$ has finite condition number $\kappa = \lmax/ \lmin$, where $\lmax$ and $\lmin$ are the largest and smallest eigenvalues of $\ATA$, \ref{heavyball} with properly chosen constant step-size and momentum parameters $\alpha_k = \alpha$ and $\beta_k = \beta$  provably attains the optimal linear rate 
\begin{equation*}
\| \vec{x}_k - \vec{x}^{*} \| \leq C_{\textup{HBM}} \left( 1 - \frac{1}{\sqrt{\kappa}} \right)^k \| \vec{x}_0 - \vec{x}^{*} \|,  \quad  C_{\textup{HBM}} > 0. 
\end{equation*}

When $\beta_k = 0,$ \ref{heavyball}  reduces to the standard gradient descent algorithm which, with the optimal choice of step-sizes $\alpha_k$, converges at a sub-optimal linear rate
\begin{equation*}
\| \vec{x}_k - \vec{x}^{*} \| \leq C_{\textup{GD}} \left( 1 - \frac{1}{\kappa} \right)^k \| \vec{x}_0 - \vec{x}^{*} \|, \quad  C_{\textup{GD}} > 0. 
\end{equation*}

On large-scale problems, computing the gradient $\nabla f(\vec{x})$ can be prohibitively expensive. 
For problems such as \cref{eqn:lls}, it is common to replace applications of the gradient with a minibatch stochastic gradient
\[
\nabla f_{S_k}(\vec{x}) = \frac{1}{B} \sum_{j \in S_k} \frac{1}{p_j} \nabla f_j(\vec{x}),
\]
where $S_k$ contains $B$ indices drawn independently with replacement from  $\{1,2,\ldots, n\}$ where, at each draw, $p_j$ is the probability that an index $j$ is chosen.
Note that this sampling strategy ensures $\E [\nabla f_{S_k}(\vec{x})] = \nabla f(\vec{x})$.

We denote by \emph{minibatch-heavy ball momentum} (Minibatch-HBM) the following algorithm: starting from initial conditions $\vec{x}_1=\vec{x}_0$, iterate until convergence \begin{equation}
\label{stoc_heavyball}
\tag{Minibatch-HBM}
    \vec{x}_{k+1} = \vec{x}_k - \alpha_k \nabla f_{S_k}(\vec{x}_k) + \beta_k (\vec{x}_k - \vec{x}_{k-1}).
\end{equation}
In the case of \cref{eqn:lls}, the minibatch stochastic gradient can be written
\begin{equation}\label{eqn:mb_gradient}
    \nabla f_{S_k}(\vec{x}_k) = \frac{1}{B} \sum_{j \in S_k} \frac{1}{p_j} \vec{a}_j(\vec{a}_j^\T \vec{x}_k - b_j),
\end{equation}
where $\vec{a}_j^\T$ is the $j$th row of $\vec{A}$ and $b_j$ is the $j$th entry of $\vec{b}$.
\begin{assumption}\label{asm:sampling_probs}
Throughout, we will assume that, for some $\eta\geq 1$, the sampling probabilities are such that
\begin{equation}
\label{eqn:sampling_probs}
    \eta  p_j \geq \frac{{\|\vec{a}_j\|^2}}{{\|\vec{A}\|_\F^2}},
    \qquad j=1,2,\ldots, n,
\end{equation}
where $\|\cdot\|_\F$ represents the matrix Frobenious norm.
\end{assumption}
\begin{remark}
If rows are sampled proportionally to their squared norm $p_j = \|\vec{a}_j\|^2/\|\vec{A}\|_\F^2$, we have $\eta = 1$.
If rows are sampled i.i.d. uniformly, $p_j = 1/n$ for $j=1,\ldots, n$, \eqref{eqn:sampling_probs} is satisfied with $\eta = n \max_j \|\vec{a}_j\|^2/\|\vec{A}\|_\F^2$.
\end{remark}

Applied to the problem\cref{eqn:lls}, plain stochastic gradient descent $(B=1,\beta=0)$ with an appropriately chosen step-size ($p_j = \| \vec{a}_j \|^2 / \|\vec{A}\|_\F^2$) is equivalent to the randomized Kaczmarz (RK) algorithm  if importance weighted sampling ($\eta=1$) is used \citep{needell2014stochastic}.
The standard version of RK extends the original cyclic Kaczmarz algorithm \citep{kaczmarz1937} and, as proved in \citep{strohmer2008randomized}, converges in a number of iterations scaling with $d\bar{\kappa}$, where $\bar{\kappa} = \lave/\lmin$ is the \emph{smoothed} condition number of $\ATA$, and $\lave$ and $\lmin$ are the average and smallest eigenvalues of $\ATA$, respectively.
Note the important relationship between the smoothed and the standard condition numbers:
\begin{equation*}
\bar{\kappa} \leq \kappa \leq d \bar{\kappa}.
\end{equation*}
This relationship implies that, when $n \geq d\sqrt{\kappa}$, RK at least matches (up to constants) the performance of \ref{heavyball}. 
If $\bar{\kappa} \ll\kappa$ or $n\gg d\sqrt{\kappa}$ then RK can significantly outperform  \ref{heavyball}, at least in terms of total number of row products.

While the number of row products of RK is reduced compared to \ref{heavyball}, the number of iterations required to converge is increased. 
In practice, running times are not necessarily directly associated with the number of row products, and instead depend on other factors such as communication and memory access patterns. 
These costs often scale with the number of iterations, so it is desirable to understand whether the iteration complexity of \ref{stoc_heavyball} can be reduced to that of \ref{heavyball}.

\begin{figure}
    \centering
    \includegraphics[width=\textwidth]{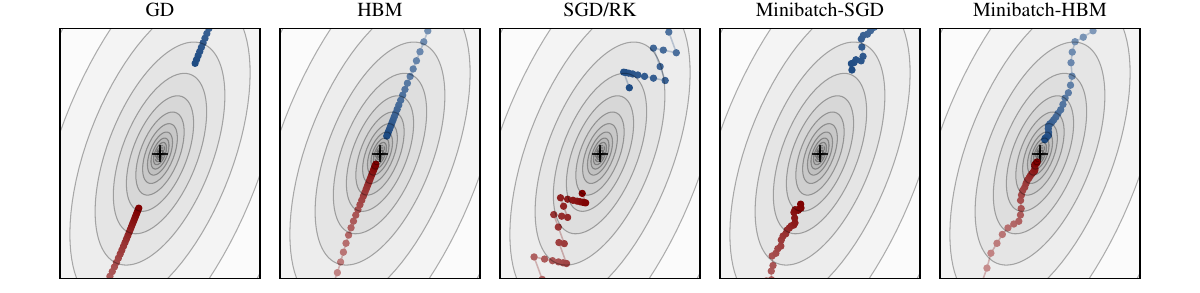}
    \caption{
    Sample convergence trajectories for various iterative methods applied to a quadratic problem  
    \cref{eqn:lls} with $n=200$ and $d=2$.
    Gradient descent with heavy ball momentum (\ref{heavyball}) allows for accelerated convergence over gradient descent (GD).
    Stochastic gradient descent (SGD) allows for lower per iteration costs, and the use of batching (Minibatch-SGD) reduces the variance of the iterates.
    While batching and momentum are often used simultaneously (\ref{stoc_heavyball}), convergence guarantees have remained elusive, even for quadratic objectives \cref{eqn:lls}.
    In this paper we prove that, on such objectives, \ref{stoc_heavyball} converges linearly at the same rate as \ref{heavyball}, provided the batch size is sufficiently large in a precise sense.
    }
    \label{fig:2d_compare}
\end{figure}

\subsection{Contributions}

We aim to make precise the observation that stochastic momentum affords acceleration in the minibatch setting.
An illustration of this phenomenon is provided in \cref{fig:2d_compare}.
Our main theoretical result is a proof that the linear rate of convergence of \ref{heavyball} can be matched by \ref{stoc_heavyball}, provided the batch size is  larger than a critical size.
Informally, this result can be summarized as follows:
\begin{theorem}
\label{thm:B_order_informal}
Consider \ref{stoc_heavyball} applied to a strongly convex quadratic objective \cref{eqn:lls} with stochastic gradients \cref{eqn:mb_gradient} whose sampling probabilities satisfy \cref{eqn:sampling_probs} with parameter $\eta \geq 1$.
Suppose that the minimizer $\vec{x}^*$ satisfies $\vec{A}\vec{x}^* = \vec{b}$.
Then, with the same fixed step-size and momentum parameters $\alpha_k = \alpha$, $\beta_k = \beta$ as \ref{heavyball}, there exists a constant $C>0$ such that, if $\kappa$ is sufficiently large and $B \geq C \eta d\log(d) \bar{\kappa} \sqrt{\kappa}$, the \ref{stoc_heavyball} iterates converge in expected norm at least at a linear rate $1-1/\sqrt{\kappa}$.
\end{theorem}
For a more precise statement of \cref{thm:B_order_informal}, see \cref{thm:B_order}.

\begin{remark}
    The bound for $B$ does not depend directly on $n$, and in many situations $B\ll n$.
    In these cases, \ref{stoc_heavyball} offers a provable speedup over \ref{heavyball}.
    \Cref{thm:linear_cons} provides a more fine-grained relationship between the momentum and step-size parameters than \cref{thm:B_order_informal}. In particular, it shows that it is always possible to take $B<n$ and have \ref{stoc_heavyball} converge similar to \ref{heavyball}, albeit at the cost of slowed convergence compared to the rate of convergence of \ref{heavyball} with the optimal $\alpha,\beta$.
\end{remark}

Owing to the equivalence between SGD on convex quadratics and the randomized Kaczmarz algorithm \citep{needell2014stochastic} our convergence guarantees give a provable iteration complexity $\sqrt{\kappa}$ for randomized Kaczmarz type algorithms.
Our analysis method is quite general, and can be used to certify fast linear convergence for various forms of momentum beyond heavy ball momentum; in \cref{sec:NAG} we illustrate the generality of the approach by proving an analogous convergence result for a minibatch Nesterov's accelerated gradient method in the setting of linear regression.

\subsection{Literature Review}

In the remainder of this section, we provide an overview of state-of-art existing results for row-sampling methods for solving (consistent) linear squares problems.
A summary is given in \cref{tab:compare}.

\begin{table}
    \caption{
    Runtime comparisons for row-sampling iterative methods for solving a consistent linear least squares problem \cref{eqn:lls} to constant accuracy when $\vec{A}\in\R^{n\times d}$ for large condition number $\kappa$.
    Constants and a logarithmic dependence on the accuracy parameter are suppressed.
    Here $\kappa$ and $\bar{\kappa}$ are the regular and smoothed condition numbers of $\ATA$. 
    Due to practical considerations such as parallelization, data movement, caching, energy efficiency, etc., the real-world cost of an iteration does not necessarily scale linearly with the number of row products. 
    }
    \label{tab:compare}
    \centering
    \begin{tabular}{rccp{6.6cm}}
    \toprule
    Algorithm & iterations &  \# row prods/iter. & references \\ \midrule
    \ref{heavyball} & $\sqrt{\kappa}$ & $n$ & \cite{hestenes1952methods}, \cite{polyak1964methods}, \cite{liesen2013krylov}  \\ 
    SGD/RK  & $d \, \bar{\kappa} $ & $1$ & \cite{strohmer2008randomized}, \cite{needell2014stochastic}\\
    ARK & $d\sqrt{\bar{\kappa}} $ & $1$ & \cite{liu2015accelerated} \\
    \ref{stoc_heavyball} & $\sqrt{\kappa}$ & $ B \gtrsim d \log(d) \bar{\kappa}\sqrt{\kappa}$ & (this paper)\\
    \bottomrule
    \end{tabular}
\end{table}

\paragraph{Randomized Kaczmarz.}
A number of improvements to the standard RK algorithm have been proposed.
Liu and Wright \citep{liu2015accelerated} introduce an accelerated randomized Kaczmarz (ARK) method which, through the use of Nesterov's acceleration, can achieve a faster rate of convergence compared to RK. 
However, their rate is still sub-optimal compared to the rate attained by \ref{heavyball}.
Moreover, ARK is less able to take advantage of potential sparsity in the data matrix $\vec{A}$ than the standard RK algorithm and \ref{stoc_heavyball}.
This issue is partially addressed by a special implementation of ARK for sparse matrices, but is still of concern for particularly sparse matrices.

Minibatching in the setting of the randomized Kaczmarz has been studied extensively in ``small'' batch regimes \citep{needell2014paved,needell2016batched,moorman2020randomized}.  
These works view minibatching as a way to reduce the variance of iterates and improve on the standard RK algorithm.
In general, the bounds for the convergence rates for such algorithms are complicated, but can improve on the convergence rate of RK by up to a factor of $B$ in the best case.
This ``best case'' improvement, however, can only be attained for small $B$; indeed, RK reduces to standard gradient descent in the deterministic gradient limit.
In contrast to these works, we study minibatching in the randomized Kaczmarz method as a \emph{necessary} algorithmic structure for unlocking the fast convergence rate of HBM.

\paragraph{Minibatch-HBM.}
Several recent works provide theoretical convergence guarantees for \ref{stoc_heavyball}. Loizou and Richt{\'a}rik \citep{loizou2017linearly, loizou2020momentum} show that \ref{stoc_heavyball} can achieve a linear rate of convergence for solving convex linear regression problems. However, the linear rate they show is slower than the rate of (deterministic)  \ref{heavyball}  in the same setting.
Gitman et al. \citep{gitman2019understanding} establish local convergence guarantees for the \ref{stoc_heavyball} method for general strongly convex functions and appropriate choice of the parameters $\alpha_k$ and $\beta_k$. 
Liu et al. \citep{liu2020improved} show that \ref{stoc_heavyball} converges as fast as SGD for smooth strongly convex and nonconvex functions.
Under the assumption that the stochastic gradients have uniformly bounded variance, Can et al. \citep{can2019accelerated} provide a number of convergence guarantees for stochastic HBM. 
In particular, it is shown that the same fast rate of convergence can be attained for full-batch quadratic objectives with bounded additive noise.  The results of Can et al. \citep{can2019accelerated} however do not apply to the setting of randomized Kaczmarz, where the variance of the stochastic gradients necessarily grows proportionally to the squared norm of the full gradient. 

Jain et al. \citep{jain2018accelerating} demonstrate that \ref{stoc_heavyball} with a batch size of $B=1$ provably fails to achieve faster convergence than SGD. They acknowledged that the favorable empirical results of \ref{stoc_heavyball}, such as those found in \citep{sutskever2013importance}, should be seen as an ``artifact" of large minibatching, where the variance of stochastic gradients is sufficiently reduced that the deterministic convergence behavior of \ref{heavyball} dominates. \emph{In this paper, we aim to precisely quantify this observation by providing a characterization of the minimal batch size required for \ref{stoc_heavyball} to achieve fast linear convergence comparable to that of \ref{heavyball}.}

In concurrent work, Lee et al. \citep{lee2022trajectory} analyze the dynamics of \ref{stoc_heavyball} applied to quadratic objectives corresponding to a general class of random data matrices. 
Their results show that when the batch size is sufficiently large, \ref{stoc_heavyball} converges like its deterministic counterpart but convergence is necessarily slower for smaller batch sizes.
The batch size requirement of \citep{lee2022trajectory} is a factor of $\kappa$ better than what we obtain (see \cref{thm:B_order_informal}). 
However, while our analysis makes no assumptions on $\vec{A}$, \citep{lee2022trajectory} requires certain invariance assumptions on the singular vectors of $\vec{A}$.
It would be interesting to understand whether the extra factor of $\kappa$ in our bound can be improved or whether is a necessary artifact of the lack of assumptions on $\vec{A}$.

\paragraph{Stochastic Nesterov's Accelerated Gradient (SNAG).}
Several recent works have analyzed the theoretical convergence properties of stochastic Nesterov's accelerated gradient (SNAG) methods and their variants in both strongly convex and nonconvex settings \citep{aybat2020robust, can2019accelerated,ghadimi2016accelerated,vaswani2019fast}.
Aybat et al. \citep{aybat2020robust} and Can et al. \citep{can2019accelerated} demonstrate accelerated convergence guarantees of SNAG method variants to a neighborhood of the solution for problems with uniformly bounded noise. Additionally, Ghadimi and Lan \citep{ghadimi2016accelerated} provide convergence guarantees for SNAG variants in nonconvex settings.
Vaswani et al. \citep{vaswani2019fast} show that SNAG methods achieve accelerated convergence rates to the solution for over-parameterized machine learning models under the assumption of the strong gradient growth condition, where the $\ell_2$ norm of the stochastic gradients is assumed to be bounded by the norm of the gradient. 
Our contributions imply that the strong gradient growth conditions hold in the consistent under-parameterized least squares setting for the stochastic minibatch gradient, and demonstrating that this condition implies acceleration for \ref{stoc_heavyball}, in addition to Nesterov momentum. 
Acceleration techniques have also been integrated with the variance reduction techniques to achieve optimal convergence rate guarantees for finite-sum problems \citep{allen2018katyusha, lin2018catalyst, defazio2016simple, frostig2015regularizing, zhou2019direct}. 

Ma et al. \citep{pmlr-v80-ma18a} established critical batch size for SGD to retain the convergence rate of deterministic gradient descent method where as in this work we establish the critical batch size for \ref{stoc_heavyball} to retain the convergence properties of \ref{heavyball}.

\subsection{Notation}
We denote vectors using lower case roman letters and matrices using upper case roman letters. 
We use $\vec{x} \in \R^d$ to denote the variables of the optimization problem and $\vec{x}^*$ to denote the minimizer of $f(\vec{x})$. 
We use $\|\cdot\|$ to represent the Euclidean norm for vectors and operator norm for matrices and $\|\cdot\|_\F$ to represent the matrix Frobenious norm.
Throughout, $\vec{A}$ will be a $n\times d$ matrix and $\vec{b}$ a length $n$ vector.
The eigenvalues of $\ATA$ are denoted $\lambda_1,\lambda_2, \ldots, \lambda_d$, and we write $\lmax$, $\lmin$, and $\lave$ for the largest, smallest, and average eigenvalue of $\ATA$.
The regular and smoothed condition numbers $\kappa$ and $\bar{\kappa}$ of $\ATA$ are respectively defined as $\kappa = \lmax/\lmin$ and $\bar{\kappa} = \lave/\lmin$. 
All logarithms are natural, and we denote complex numbers by i and Euler's number $2.718\ldots$ by $\mathrm{e}$.

\section{Preliminaries}
In this section we provide an overview of \ref{heavyball} analysis and important statements from random matrix theory that are used in proving the convergence of \ref{stoc_heavyball}.
\subsection{Standard analysis of heavy ball momentum for quadratics}
\label{sec:HBM}

We review the standard analysis for heavy ball momentum (\ref{heavyball}) in the setting of strongly convex quadratic optimization problems \cref{eqn:lls} (see \citep{recht2012lyapunov}).
Here and henceforth, we take $\alpha_k = \alpha$ and $\beta_k = \beta$ as constants.

First, we re-write the  \ref{heavyball}  updates as
\begin{equation}
\label{eqn:HBM}
    \vec{x}_{k+1} = \vec{x}_k - \alpha (\ATA \vec{x}_{k} - \vec{A}^\T\vec{b}) + \beta (\vec{x}_k - \vec{x}_{k-1})
\end{equation}
so, by definition, the  \ref{heavyball}  updates satisfy
\begin{align}
\vec{x}_{k+1}- \vec{x}^{*} &= \vec{x}_{k} - \alpha (\ATA \vec{x}_{k} - \ATA \vec{x}^{*}) + \beta (\vec{x}_{k} -\vec{x}_{k-1}) - \vec{x}^{*} \nonumber \\
&= \vec{x}_{k} - \vec{x}^{*} - \alpha \ATA (\vec{x}_{k} - \vec{x}^{*}) + \beta (\vec{x}_{k} -\vec{x}_{k-1}) \nonumber \\
&= (\vec{I} - \alpha \ATA) (\vec{x}_{k} - \vec{x}^{*}) + \beta (\vec{x}_{k} -\vec{x}_{k-1} - \vec{x}^{*} + \vec{x}^{*}) \nonumber \\
&= \big( (1 + \beta) \vec{I} - \alpha \ATA \big) ( \vec{x}_{k} - \vec{x}^{*}) - \beta (\vec{x}_{k-1} - \vec{x}^{*}). \nonumber
\end{align}
This can be written more concisely as
\begin{equation}
\label{define_T}
\begin{bmatrix}\vec{x}_{k+1} - \vec{x}^{*} \\ \vec{x}_{k} - \vec{x}^{*} \end{bmatrix} 
=  \underbrace{ \begin{bmatrix} (1+ \beta ) \vec{I} - \alpha \ATA & - \beta \vec{I} \\
\vec{I} & \vec{0} \end{bmatrix}}_{\vec{T} = \vec{T}(\alpha, \beta) } \begin{bmatrix} \vec{x}_{k} - \vec{x}^{*} \\\vec{x}_{k-1} - \vec{x}^{*} \end{bmatrix},
\end{equation}
where $\vec{T}$ is the transition matrix taking us from the error vectors at steps $k$ and $k-1$ to the error vectors at steps $k+1$ and $k$.
Repeatedly applying this relation we find
\begin{equation}
\begin{bmatrix}\vec{x}_{k+1} - \vec{x}^{*} \\ \vec{x}_{k} - \vec{x}^{*} \end{bmatrix} 
=  \vec{T}^k \begin{bmatrix} \vec{x}_{1} - \vec{x}^{*} \\\vec{x}_{0} - \vec{x}^{*} \end{bmatrix}
\end{equation}
from which we obtain the error bound
\begin{equation*}
\left\| \begin{bmatrix}\vec{x}_{k+1} - \vec{x}^{*} \\ \vec{x}_{k} - \vec{x}^{*} \end{bmatrix} \right\| 
\leq \|  \vec{T}^k \| \left\| \begin{bmatrix} \vec{x}_1 - \vec{x}^{*} \\ \vec{x}_{0} - \vec{x}^{*} \end{bmatrix} \right\|.
\end{equation*}
The assumption $\vec{x}_1 = \vec{x}_0$ allows us to write
\begin{equation*}
    {\left\| \vec{x}_{k+1} - \vec{x}^{*}  \right\|}
    \leq \sqrt{2} \|  \vec{T}^k \| {\|\vec{x}_{0} - \vec{x}^{*}\|}.
\end{equation*}

The difficulty in analyzing \ref{heavyball} compared to plain gradient descent (when $\beta = 0$) lies in the fact $\|\vec{T}\|^k$ need not provide a useful upper bound for $\|\vec{T}^k\|$.
Indeed, while $\|\vec{T}^k\| = \| \vec{T}\|^k$ for symmetric matrices, this is not necessarily the case for non-symmetric matrices.
To get around this issue, it is common to bound the spectral radius $\rho(\vec{T}) = \max_j \{ | \lambda_j | \}$ and use Gelfand's formula $\rho(\vec{T}) = \lim_{k \rightarrow \infty} \| \vec{T}^k \|^{1/k}$ to derive the rate of convergence \citep{recht2012lyapunov}.

To bound the spectral radius of $\vec{T}$, note that $\vec{T}$ is orthogonally similar to a block diagonal matrix consisting of $2 \times 2$ components $\{\vec{T}_j\}$. 
Specifically, 
\begin{equation}
\label{block_diag}
\vec{U}^{-1} \vec{T} \vec{U} =  \begin{bmatrix} \vec{T}_1 & & \\ & \vec{T}_2 & & \\ & & \ddots\\ & & & \vec{T}_n \end{bmatrix}, \quad \text{where}  \quad\vec{T}_j = \begin{bmatrix} 1 + \beta - \alpha \lambda_j & - \beta \\ 1 & 0 \end{bmatrix}
\end{equation}
for each $j=1,2,\ldots, n$,  $\vec{U}$ is a certain orthogonal matrix (see \citep{recht2012lyapunov}), and $\{ \lambda_j \}$ are the eigenvalues of $\ATA$. 
For each $j=1,2,\ldots, n$, the eigenvalues $z_j^{\pm}$ of $\vec{T}_j$ are easily computed to be
\begin{equation*}
    z_j^\pm := \frac{1}{2} \left( 1+ \beta - \alpha \lambda_j  \pm \sqrt{(1+ \beta - \alpha \lambda_j)^2-4\beta} \right)
\end{equation*}
and are therefore non-real if and only if
\begin{equation}
    \label{complex_eigs}
    (1+ \beta - \alpha \lambda_j)^2-4\beta < 0.
\end{equation}
In this case, the magnitude of both the eigenvalues of $\vec{T}_j$ is 
\begin{equation*}
    |z_j^\pm| 
    = \frac{1}{2} \sqrt{ (1+ \beta - \alpha \lambda_j)^2  + |(1+ \beta - \alpha \lambda_j)^2-4\beta | }
    = \sqrt{\beta}.
\end{equation*}
Here we have used that $|(1+ \beta - \alpha \lambda_j)^2-4\beta | = 4\beta - (1+ \beta - \alpha \lambda_j)^2$ whenever \cref{complex_eigs} holds.
Therefore, provided
\begin{equation}
    \label{complex_eigs_all}
    (1+ \beta - \alpha \lambda_j)^2-4\beta < 0 \quad \mbox{for all} \quad j= 1,2,\ldots,n,
\end{equation}
we have that
\begin{equation}
    \label{spec_radius}
    \rho(\vec{T}) = \max \{ |z_{j}^\pm| : j=1,\ldots, d \} = \sqrt{\beta}.
\end{equation}

We would like to choose $\sqrt{\beta} = \rho(\vec{T})$ as small as possible subject to the condition that \cref{complex_eigs_all} holds.\footnote{If \cref{complex_eigs_all} does not hold, then the larger of $|z_j^\pm|$ will be greater than $\sqrt{\beta}$.}
Note that \cref{complex_eigs_all} is equivalent to the condition
\begin{equation*}
    \frac{(1-\sqrt{\beta})^2}{\lambda_j}
    < \alpha < 
    \frac{(1+\sqrt{\beta})^2}{\lambda_j}
    \quad \mbox{for all} \quad j= 1,2,\ldots,n,
\end{equation*}
which we can rewrite as
\begin{equation}
\label{eqn:alpha_beta_cond}
    \frac{(1-\sqrt{\beta})^2}{\lmin}
    < \alpha < 
    \frac{(1+\sqrt{\beta})^2}{\lmax}.
\end{equation}
Minimizing in $\beta$ gives the condition ${(1-\sqrt{\beta})^2}/{\lmin} = \alpha = {(1+\sqrt{\beta})^2}/{\lmax}$ from which we determine
\begin{equation} \label{eqn:alphabeta}
    \sqrt{\alpha^*} = \frac{2}{\sqrt{\lmax}+\sqrt{\lmin}}
    \quad \text{and} \quad
    \sqrt{\beta^*} = \frac{\alpha (\lmax - \lmin)}{4} 
    = \frac{\sqrt{\kappa}-1}{\sqrt{\kappa}+1}.
\end{equation}
As noted at the start of this section, this gives an asymptotic rate of convergence $\sqrt{\beta}$.

\subsection{A closer look at the quadratic case}

To derive a bound for stochastic \ref{heavyball} at finite $k$, it is desired to understand the eigendecomposition of the matrix $\vec{T}$ from \cref{block_diag} more carefully.
Thus, we might aim to diagonalize $\vec{T}$ as
\begin{equation}
\label{TtoC}
\vec{T} = \vec{U}\vec{C} \vec{D}\vec{C}^{-1}\vec{U}^{-1} = (\vec{U}\vec{C}) \vec{D} (\vec{U}\vec{C})^{-1},
\end{equation}
where $\vec{U}$ is the previously described orthogonal matrix rotating $\vec{T}$ into block diagonal form \cref{block_diag} and $\vec{D}$ and $\vec{C}$ are block-diagonal matrices with $2 \times 2$ blocks $\{\vec{D}_j\}$ and $\{\vec{C}_j\}$ where, for each $j=1,\ldots, d$,  $\vec{T}_j$ is diagonalized as $\vec{T}_j = \vec{C}_j \vec{D}_j \vec{C}_j^{-1}$.
Given such a factorization, we would have $\vec{T}^k = (\vec{U}\vec{C}) \vec{D}^k (\vec{U}\vec{C})^{-1}$. Then, since that $\vec{U}$ is unitary, we would obtain the bound
\begin{equation*}
    \|\vec{T}^k\| \leq \condCbd(\alpha,\beta) \|\vec{D}\|^k = \condCbd(\alpha,\beta) (\sqrt{\beta})^{k},
\end{equation*} 
where $\condCbd(\alpha,\beta) = \|\vec{C}\|\|\vec{C}^{-1}\| = \|(\vec{U}\vec{C})\|\|(\vec{U}\vec{C})^{-1}\|$ is the condition number of the eigenvector matrix $\vec{U}\vec{C}$.

However, if $\alpha$ and $\beta$ are chosen as in \cref{eqn:alphabeta}, $\vec{T}$ is defective (that is, does not have a complete basis of eigenvectors) and no such diagonalization exists.\footnote{In particular, the blocks $\vec{T}_j$ corresponding to the smallest and largest eigenvalues each have only a single eigenvector.}
To avoid this issue, we perturb the choices of $\alpha$ and $\beta$, and define, for some $\gamma \in (0,\lmin)$,
\begin{align}\label{eq:params_L}
    L  = \lmax + \gamma
    \quad \text{and} \quad
    \ell = \lmin- \gamma.
\end{align}
Taking
\begin{equation}\label{eq:params_alpha}
    \sqrt{\alpha} = 
    \frac{2}{\sqrt{L}+\sqrt{\ell}}
    \quad \text{and} \quad
    \sqrt{\beta} 
    = \frac{\alpha (L - \ell)}{4} 
    = \frac{\sqrt{L/\ell}-1}{\sqrt{L/\ell}+1}
\end{equation}
ensures that \cref{complex_eigs_all} holds and that $\vec{T}$ is diagonalizable.
Indeed, since we can write $z_j^\pm = a_j \pm i b_j$ for $a_j, b_j \in \R$ with $b_j\neq 0$, it is easily verified that the (up to a scaling of the eigenvectors) eigendecomposition $\vec{T}_j \vec{C}_j = \vec{C}_j \vec{D}_j$ for $\vec{T}_j$ is
\begin{equation}\label{eqn:eigdecomp}
    \vec{T}_j \begin{bmatrix}
    a_j +i b_j & a_j-i b_j
    \\
    1&1
    \end{bmatrix}
    = 
    \begin{bmatrix}
    a_j +i b_j & a_j-i b_j
    \\
    1&1
    \end{bmatrix}
    \begin{bmatrix}
    a_j + i b_j \\ & a_j - i b_j
    \end{bmatrix}.
\end{equation}
We clearly have $\| \vec{D} \| = \max_j | z_j^\pm | = \sqrt{\beta}$, so the \ref{heavyball} iterates satisfy the convergence guarantee 
\begin{equation}
    \label{eqn:heavyball_bd}
    {\left\|\vec{x}_{k+1} - \vec{x}^{*} \right\|}
    \leq \sqrt{2} \condCbd(\alpha,\beta) \left( \frac{\sqrt{L/\ell} - 1}{\sqrt{L/\ell} + 1} \right)^{k} {\left\|  \vec{x}_0 - \vec{x}^{*} \right\|}.
\end{equation}

Note that $\condCbd(\alpha,\beta)$, $L$, and $\ell$ each depend on $\lmax$, $\lmin$, and $\gamma$.
The dependency of $\condCbd(\alpha,\beta)$ on these values is through $\vec{C}$, which, up to a unitary scaling by $\vec{U}$, is the eigenvector matrix of the transition matrix $\vec{T}$.
We can bound $\condCbd(\alpha,\beta) $ by the following lemma.
\begin{lemma}
\label{thm:HBM_condno}
For any $\gamma \in (0,\lmin)$ set $\ell = \lmin-\gamma$ and $L =\lmax+\gamma$ and choose $\alpha = {4}/{(\sqrt{L} + \sqrt{\ell})^2}$ and $\sqrt{\beta} = \alpha(L-\ell)/4$.
Let $\condCbd(\alpha,\beta)  = \| \vec{U}\vec{C} \| \| (\vec{U}\vec{C})^{-1} \|$, where $\vec{U}\vec{C}$ is the eigenvector matrix for $\vec{T}$.
Then
\begin{equation*}
    \condCbd(\alpha,\beta) \leq \frac{4}{\alpha \sqrt{\gamma (\gamma+\lmax - \lmin)}}.
\end{equation*}
\end{lemma}

\begin{proof}
In order to bound $\condCbd(\alpha,\beta) $, we note that the block diagonal structure of $\vec{C}$ implies that $\|\vec{C}\| = \max \{ \|\vec{C}_j\| : j=1,\ldots, d \}$ and $\| \vec{C}^{-1}\| = \max \{ \|\vec{C}_j^{-1}\| : j=1,\ldots, d \}$.

By construction, the specified values of $\alpha$ and $\beta$ ensure that $(1+\beta - \alpha\lambda_i)^2 < 4\beta$ for all $i=1,2\dots,d$; i.e. condition \cref{complex_eigs}.
This implies $|z_j^\pm| = \sqrt{\beta}$, so we easily compute
\begin{equation*}
    \|\vec{C}_j\|^2 \leq \|\vec{C}_j\|_\F^2 
    = |z_j^+|^2+|z_j^-|^2 +1+1 
    = 2\beta + 2 \leq 4.
\end{equation*}
By direct computation we find
\begin{equation*}
    \vec{C}_j^{-1} =
    \begin{bmatrix}
    a_j +i b_j & a_j-i b_j \\
    1&1
    \end{bmatrix}^{-1} =
    \frac{1}{2i b_j}
    \begin{bmatrix}
    1 & -a_j + i b_j \\
    -1 & a_j + i b_j 
    \end{bmatrix}.
\end{equation*}
To bound $\|\vec{C}_j^{-1}\|$ we first note that the condition \cref{complex_eigs} is also equivalent to
\begin{equation*}
    \ell = \frac{(1-\sqrt{\beta})^2}{\alpha}
    < \lambda_j <
    \frac{(1+\sqrt{\beta})^2}{\alpha} = L
\end{equation*}
which implies that
\begin{equation*}
4 \beta - (1+ \beta - \alpha \lambda_j)^2
= \alpha^2 (\lambda_j - \ell)(L-\lambda_j)
\geq \alpha^2 \gamma (L-\lmin).
\end{equation*}
Since $L \geq \lmax$ we therefore have the bound
\begin{equation*}
    \|\vec{C}_j^{-1} \|^2 \leq \|\vec{C}_j^{-1} \|_\F^2 
    = \frac{\|\vec{C}_j\|_\F^2}{4 |b_j|^2}
    = \frac{2 (1 + \beta)}{4\beta - (1+\beta - \alpha \lambda_j)^2}
    \leq \frac{4}{\alpha^2 \gamma (L - \lmin)}.
\end{equation*}
The result follows by combining the above expressions.
\end{proof}
\subsection{Lemmas from non-asymptotic random matrix theory}

Before we prove our main result, we need to introduce tools from non-asymptotic random matrix theory which are crucial components of the proof.

\begin{proposition}
\label{thm:matrixmoment}
Consider a finite sequence $\{ \vec{W}_k \}$ of independent random matrices with common dimension $d_1 \times d_2$.
Assume that
\begin{equation*}
    \E[ \vec{W}_i ] = \vec{0} \quad \quad \text{and} \quad \quad \| \vec{W}_i \| \leq W \quad \text{for each index } i
\end{equation*}
and introduce the random matrix 
\begin{equation*}
    \vec{Z} = \vec{W}_1 + \cdots + \vec{W}_k.
\end{equation*}
Let $v(\vec{Z})$ be the matrix variance statistic of the sum: 
\begin{align*}
v(\vec{Z}) 
&= \max\big\{ \textstyle \big\| \sum_i \E[ \vec{W}_i \vec{W}_i^{\T} ] \big\|, \big\| \sum_i \E[ \vec{W}_i^{\T} \vec{W}_i ] \big\|  \big\}.
\end{align*}
Then, 
\begin{align*}
    \E\big[ \| \vec{Z} \| \big] &\leq \sqrt{2 v(\vec{Z}) \log(d_1+d_2)} + \frac{1}{3} W \log(d_1+d_2),
    \\
    \sqrt{\E[ \| \vec{Z} \|^2 ]} &\leq \sqrt{2 \mathrm{e} v(\vec{Z}) \log(d_1+d_2)} + 4 \mathrm{e} W \log(d_1+d_2).
\end{align*}
\end{proposition}

The bound on $\E[\|\vec{Z}\|]$ is Theorem 6.1.1 in \citep{tropp2015introduction}, and the bound on $\sqrt{ \E[ \| \vec{Z} \|^2 ]}$ follows from equation 6.1.6 in  \citep{tropp2015introduction} and the fact $\sqrt{\E[ \max_i \| \vec{W}_i \|^2]} \leq W$.
Equation 6.1.6 in \citep{tropp2015introduction} comes from applying Theorem A.1 in \citep{chen2012masked} to the Hermitian dilation  of $\vec{Z}$. 
Under the stated conditions, the logarithmic dependence on the dimension is necessary \citep{tropp2015introduction}.

We will also use a theorem on products of random matrices from \citep{huang2021matrix}:
\begin{proposition}[Corollary 5.4 in \citep{huang2021matrix}]
\label[proposition]{prop:products}
Consider an independent sequence of $d \times d$ random matrices $\vec{X}_1, \dots, \vec{X}_{k}$, and form the product
\begin{equation*}
    \vec{Z} =\vec{X}_{k}\vec{X}_{k-1} \cdots \vec{X}_1.
\end{equation*}
Assume $\| \mathbb{E} [\vec{X}_i] \| \leq q_i$ and $ \mathbb{E}[ \| \vec{X}_i - \mathbb{E} \vec{X}_i \|^2] ^{1/2} \leq \sigma_i q_i$ for $i = 1, \ldots, k$.
Let $Q = \prod_{i=1}^n q_i$ and $v = \sum_{i=1}^k \sigma_i^2$.  Then
\begin{equation*}
    \mathbb{E} \big[\| \vec{Z} \|\big] \leq Q \exp{ \left( \sqrt{2v \max\{2v, \log(d)\}} \right)} .
\end{equation*}

\end{proposition}

\section{Main results}

We are now prepared to analyze \ref{stoc_heavyball} applied to strongly convex least squares problems of the form \cref{eqn:lls}. 
We begin by considering the case of consistent linear systems; i.e. systems for which $\vec{b}$ is in the column span of $\vec{A}$. 
In \cref{sec:inconsistent} we then provide an analogous result for inconsistent least squares problems.

We begin with a useful technical lemma which bounds the batch size required to ensure that a certain random matrix is near it's expectation.
\begin{lemma} \label{thm:sqnormbd}
Define $\vec{W}_j = B^{-1} (-p_j^{-1} \vec{a}_j\vec{a}_j^\T +\ATA)$ and let
\begin{equation*}
    \vec{W}  = \sum_{j\in S} \vec{W}_j
\end{equation*}
where $S$ is a list of $B$ indices each chosen independently according to \cref{eqn:sampling_probs}.
Then, $\sqrt{\E[\|\vec{W}\|^2]} \leq \delta$ provided
\begin{equation*}
    B \geq {8 \mathrm{e}\eta \log(2d)} \max\big\{ \| \vec{A} \|_\F^2  \|\vec{A}\|^2 \delta^{-2} , ( 4 \| \vec{A} \|_\F^4 \delta^{-2} )^{1/2} \big\}.
\end{equation*}
\end{lemma}

\begin{proof}
Since the sampling probabilities satisfy \cref{eqn:sampling_probs}, $\|\vec{a}_j \vec{a}_j^\T \| = \|\vec{a}_j\|^2 \leq  \eta p_j \|\vec{A} \|_\F^2$.
Then, since $\eta \geq 1$,
\begin{equation*}
\left \| \vec{W}_j \right\| 
\leq  \frac{1}{B} \left(\frac{1}{p_j}\| \vec{a}_j \vec{a}_j^\T \| + \| \vec{A}\|^2 \right)
\leq \frac{\eta \|\vec{A}\|_\F^2+\|\vec{A}\|^2}{B}
\leq 
\frac{2 \eta \|\vec{A} \|_\F^2}{B}.
\end{equation*}
Next, observe that
\begin{equation}
\label{eqn:W_norm}
    \vec{W}_j^\T \vec{W}_j
    = \frac{1}{B^2}  \bigg( \frac{\| \vec{a}_j \|^2}{p_j^2} \vec{a}_j\vec{a}_j^\T
    - \frac{1}{p_j} \vec{a}_j\vec{a}_j^\T \ATA
    - \frac{1}{p_j} \ATA \vec{a}_j\vec{a}_j^\T + 
    (\ATA)^2 \bigg),
\end{equation}
Using that $\E[(p_j)^{-1} \vec{a}_j \vec{a}_j^\T] = \ATA$ and $\| \vec{a}_j\|^2 \leq \eta p_j \| \vec{A}\|_\F^2$, we find that
\begin{equation*}
    \E \big[ \vec{W}_j^\T \vec{W}_j \big]
    = \frac{1}{B^2} \bigg( \sum_{i=1}^{n} \frac{\|\vec{a}_i\|^2}{p_i} \vec{a}_i\vec{a}_i^\T - (\ATA)^2 \bigg)
    \preceq \frac{1}{B^2} \big( \eta \| \vec{A} \|_\F^2 \ATA - (\ATA)^2 \big).
\end{equation*}
Here we write $\vec{M}_1 \preceq \vec{M}_2$ if $\vec{M}_2 - \vec{M}_1$ is positive semi-definite.
Note that $\vec{M}_1 \preceq \vec{M}_2$ implies the largest eigenvalue of $\vec{M}_2$ is greater than the largest eigenvalue of $\vec{M}_1$.
Therefore, using that $\vec{0} \preceq \E \big[ \vec{W}_j^\T \vec{W}_j \big]$ followed by the fact $\vec{0} \preceq \ATA \preceq \|\vec{A}\|_\F^2 \vec{I}$,
\begin{equation*}
    \big\| \E \big[ \vec{W}_j^\T \vec{W}_j \big] \big\|
    \leq \frac{1}{B^2} \big\| (\eta \| \vec{A} \|_\F^2 \vec{I} - \ATA) \ATA \big\|
    \leq \frac{\eta \| \vec{A} \|_\F^2 \| \vec{A} \|^2}{B^2} .
\end{equation*}
Thus, since $\vec{W}_j$ is symmetric and the samples in $S$ are iid, we obtain a bound for the variance statistic
\begin{equation}
    \label{eqn:vW_norm}
    v(\vec{W}) 
    = \bigg\| \sum_{j\in S} \E \big[ \vec{W}_j^\T \vec{W}_j \big] \bigg\|
    \leq  \sum_{j\in S} \bigg\| \E \big[ \vec{W}_j^\T \vec{W}_j \big] \bigg\|
    \leq \frac{\eta \| \vec{A} \|_\F^2 \| \vec{A} \|^2}{B} .
\end{equation}

Together with \cref{eqn:W_norm,eqn:vW_norm}, \cref{thm:matrixmoment} implies
\begin{equation}
    \label{eqn:EnormW2}
    \sqrt{\E\big[\|\vec{W}\|^2\big]}
    \leq \left( \frac{2 \mathrm{e} \eta \| \vec{A} \|_\F^2 \| \vec{A} \|^2 \log(2d)}{B} \right)^{1/2} +  \frac{4 \mathrm{e} (2 \eta \|\vec{A}\|_\F^2) \log(2d)}{B}.
\end{equation}
The first term is bounded by $\delta/2$ when 
\begin{equation*}
    B \geq 8 \mathrm{e} \eta \| \vec{A} \|_\F^2 \| \vec{A} \|^2 \log(2 d) \delta^{-2}
\end{equation*}
whereas the second term is bounded by $\delta/2$ when
\begin{equation*}
    B \geq 16 \mathrm{e} \eta \| \vec{A} \|_\F^2  \log(2 d)  \delta^{-1}.
\end{equation*}
The result follows by taking the max of these quantities.
\end{proof}

Our main result is the following theorem.
\begin{theorem} \label{thm:linear_cons}
Consider \ref{stoc_heavyball} applied to a strongly convex quadratic objective \cref{eqn:lls} with stochastic gradients \cref{eqn:mb_gradient} whose sampling probabilities satisfy \cref{eqn:sampling_probs}.
Fix parameters $\alpha$ and $\beta$ satisfying ${(1-\sqrt{\beta})^2}/{\lmin}
    < \alpha < 
    {(1+\sqrt{\beta})^2}/{\lmax}$. 
For any $k^*>1$  choose
\begin{equation*}
    B 
    \geq 16 \mathrm{e} \eta \log(2d) \max\left\{ \frac{\| \vec{A} \|_\F^2 \| \vec{A} \|^2\alpha^2 \condCbd(\alpha,\beta)^2 k^*}{ \beta \log(k^*)}, \left( \frac{2\| \vec{A} \|_\F^4 \alpha^2 \condCbd(\alpha,\beta)^2 k^*}{\beta \log(k^*)} \right)^{1/2} \right\}.
\end{equation*}
Then , for all $k>0$, assuming that the minimizer $\vec{x}^*$ satisfies $\vec{A}\vec{x}^* = \vec{b}$, the \ref{stoc_heavyball} iterates satisfy
\begin{equation*}
\E \big[ {\left\|\vec{x}_{k} -\vec{x}^{*} \right\|} \big]
\leq \sqrt{2} \condCbd(\alpha,\beta) \max\{d,(k^*)^{k/k^*}\} ( \sqrt{\beta})^{k} {\left\| \vec{x}_0 -\vec{x}^{*}\right\|},
\end{equation*}
where $\condCbd(\alpha,\beta)$ is the condition number of eigenvector matrix $\vec{U}\vec{C}$  for $\vec{T}$ defined in \cref{TtoC}.
\end{theorem}

\begin{remark}
For $k \leq k^*$, $(k^*)^{k/k^*} \leq \max\{\mathrm{e}, k\}$.
Thus, \ref{stoc_heavyball} converges at at rate nearly $\sqrt{\beta}$ for $k\leq \max\{d,k^*\}$. For $k>\max\{d,k^*\}$, the convergence is still linear for $k^*$ sufficiently large, but at a rate slower than $\sqrt{\beta}$.
Specifically, \ref{stoc_heavyball} converges at a rate $(\sqrt{\beta})^{1-\delta}$ where $\delta = 2\log(k^*)/(k^*\log(1/\beta))$.
\end{remark}

\begin{proof}

Due to assumption of consistency, we have that $\vec{A}\vec{x}^* = \vec{b}$.
Therefore, we can write the minibatch gradient \cref{eqn:mb_gradient} as 
\begin{equation}
\label{simple:minibatch}
\nabla f_{S_k}(\vec{x}_{k}) = \frac{1}{B} \sum_{j \in S_k} \frac{1}{p_j} \vec{a}_j  \vec{a}_j^\T ( \vec{x}_{k} - \vec{x}^{*}).
\end{equation}
Define the random matrix 
\begin{equation*}
    \vec{M}_{S_k} = \frac{1}{B} \sum_{j \in S_k} \frac{1}{p_j} \vec{a}_j \vec{a}_j^\T
\end{equation*}
and note that $\E[\vec{M}_{S_k}] = \ATA$.
Then, analogously to \cref{define_T}, the \ref{stoc_heavyball} iterates satisfy the recurrence
\begin{align*}
    \begin{bmatrix}\vec{x}_{k+1} - \vec{x}^{*} \\ \vec{x}_{k} - \vec{x}^{*} \end{bmatrix} &=  \underbrace{ \begin{bmatrix} (1+ \beta ) \vec{I} - \alpha \vec{M}_{S_k} & - \beta \vec{I} \\
\vec{I} & \vec{0} \end{bmatrix}}_{\vec{Y}_{S_k} = \vec{Y}_{S_k}(\alpha, \beta) } \begin{bmatrix} \vec{x}_{k} - \vec{x}^{*} \\\vec{x}_{k-1} - \vec{x}^{*} \end{bmatrix},
\end{align*}
where $\vec{Y}_{S_k}$ is the stochastic transition matrix at iteration $k$.
After $k$ iterations, the error satisfies 
\begin{align*}
\bigg\|\begin{bmatrix}
\vec{x}_{k+1}-\vec{x}^{*} \\
\vec{x}_k -\vec{x}^{*}
\end{bmatrix}\bigg\| &\leq \| \vec{Y}_{S_k} \vec{Y}_{S_{k-1}}\cdots \vec{Y}_{S_1} \| \bigg\|\begin{bmatrix}
\vec{x}_{1} -\vec{x}^{*} \\
\vec{x}_0 -\vec{x}^{*}
\end{bmatrix}\bigg\|, 
\end{align*}
so our goal is to bound the norm of the random matrix $\vec{Y}_{S_k} \vec{Y}_{S_{k-1}}\cdots \vec{Y}_{S_1}$.

This is a product of random matrices, so we may hope to apply \cref{prop:products}.
However, while $\E[ \vec{Y}_{S_i} ]= \vec{T}$, where $\vec{T}$ is the deterministic transition matrix \cref{define_T}, $\|\E[ \vec{Y}_{S_i}]\|$ is not necessarily bounded by $\sqrt{\beta}$.
Thus, to apply \cref{prop:products} we will instead consider
\begin{equation}
\label{Zmat}
    \vec{Z}_{k} = (\vec{U}\vec{C})^{-1} \vec{Y}_{S_k} \vec{Y}_{S_{k-1}}\cdots \vec{Y}_{S_1} (\vec{U}\vec{C})
    = \vec{X}_{S_k} \vec{X}_{S_{k-1}} \cdots \vec{X}_{S_1} 
\end{equation}
where $\vec{X}_{S_i} = (\vec{U}\vec{C})^{-1} \vec{Y}_{S_i} (\vec{U}\vec{C})$ and $\vec{U}$ and $\vec{C}$ are the matrices from \cref{TtoC}.
Then, as desired,
\begin{equation*}
\| \E  [\vec{X}_{S_i}]  \|  
= \| (\vec{U}\vec{C})^{-1} \E[ \vec{Y}_{S_i}] (\vec{U}\vec{C})  \|  
= \| (\vec{U}\vec{C})^{-1} \vec{T} (\vec{U}\vec{C})  \| 
= \| \vec{D} \|
= \sqrt{\beta}. 
\end{equation*}
Thus, if we can guarantee that the variances $\{ \sqrt{\E[ \| \vec{X}_{S_i} - \E [\vec{X}_{S_i}] \|^2 ]} \}$ are not too large, we can apply \cref{prop:products} to obtain a rate similar to \ref{heavyball}.

Towards this end, note that
\begin{equation*}
\vec{Y}_{S_i} - \E[ \vec{Y}_{S_i} ]
= \sum_{j\in S_i} \frac{\alpha}{B} \begin{bmatrix} - p_j^{-1} \vec{a}_{j} \vec{a}_{j}^\T + \ATA & \vec{0} \\ \vec{0} & \vec{0} \end{bmatrix} 
= \alpha \sum_{j \in S_i } \begin{bmatrix} \vec{W}_j & \vec{0} \\ \vec{0} & \vec{0} \end{bmatrix},
\end{equation*}
where $\vec{W}_j$ is as in \cref{thm:sqnormbd}.
This and the fact that $\| \vec{X}_{S_i} - \E[\vec{X}_{S_i}] \| = \| (\vec{U}\vec{C})^{-1}( \vec{Y}_{S_i} - \E[\vec{Y}_{S_i}] )(\vec{U}\vec{C}) \|$ implies
\begin{equation}
    \label{eqn:vY_bound}
    \sqrt{\E \big[ \| \vec{X}_{S_i} - \E [\vec{X}_{S_i} ] \|^2 \big] }
    \leq \condCbd(\alpha,\beta) \sqrt{\E\big[ \| \vec{Y}_{S_i} - \E [\vec{Y}_{S_i}] \|^2 \big]}
    \leq \alpha \condCbd(\alpha,\beta)\sqrt{\E \big[ \| \vec{W} \|^2 \big]},
\end{equation}
where $\vec{W} = \sum_{j\in S_j} \vec{W}_j$.
Using \cref{thm:sqnormbd,eqn:vY_bound}, we have $\sqrt{\E\|\vec{X}_{S_i} - \E \vec{X}_{S_i}\|^2} \leq \delta$ provided that the batch size $B$ satisfies
\begin{equation*}
    B \geq {8 \mathrm{e}\eta \log(2d)} \max\big\{ \| \vec{A} \|_\F^2  \|\vec{A}\|^2 \alpha^2 \condCbd(\alpha,\beta)^2 \delta^{-2} , ( 4 \| \vec{A} \|_\F^4 \alpha^2 \condCbd(\alpha,\beta)^2 \delta^{-2} )^{1/2} \big\}.
\end{equation*}
Applying \cref{prop:products} to the product \cref{Zmat} with the parameters
\begin{equation*}
    q_i = \sqrt{\beta},
    \quad
    \sigma_i = \delta/\sqrt{\beta},
    \quad \text{and} \quad
    v = \textstyle \sum_{i=1}^k  \sigma_i^2 = k \delta^2 / \beta
\end{equation*}
gives the bound
\begin{equation*}
\E \big[ \| \vec{Z}_k\|  \big]
\leq (\sqrt{\beta})^{k} \exp{\left( \sqrt{2v\max\{2v, \log(d)\}} \right)}.
\end{equation*}
Set $\delta^2 = \beta \log(k^*)/(2k^*)$ so that $2v = k\log(k^*)/k^*$.
This gives the desired expression for $B$.
Moreover, we then have
\begin{equation*}
    \E \big[ \| \vec{Z}_k\| \big]
    \leq (\sqrt{\beta})^{k} \exp\left( \max\{2v, \log(d)\} \right)
    = (\sqrt{\beta})^{k} \max\{ (k^*)^{k/k^*}, d \}.
\end{equation*}
Thus, we find that
\begin{equation*}
    \E\left[\| \vec{Y}_{S_k} \vec{Y}_{S_{k-1}}\cdots \vec{Y}_{S_1}\|\right]
    = \E\left[\| (\vec{U}\vec{C}) \vec{Z}_k (\vec{U}\vec{C})^{-1} \|\right]
    \leq \condCbd(\alpha,\beta) (\sqrt{\beta})^{k} \max\{ (k^*)^{k/k^*}, d \},
\end{equation*}
giving the desired bound for the iterates.
\end{proof}

The expressions for the required batch size in \cref{thm:linear_cons} is somewhat complicated, but can be simplified in the large condition number limit.
\begin{corollary}\label{thm:B_order}
Fix $c\in(0,2)$.
There exist parameters $\alpha$ and $\beta$ and a constant $C>0$ (depending on $c$) such that, for all $\kappa$ sufficiently large,
the \ref{stoc_heavyball} iterates converge in expected norm at least at a linear rate $1 - {c}/{\sqrt{\kappa}}$ provided that $B \geq C \eta d\log(d) \bar{\kappa} \sqrt{\kappa}$.
\end{corollary}
\begin{proof}
Suppose $\gamma = c_1 \lmin$ for $c_1\in(0,1)$.
Then, using that the definitions of $L$ and $\ell$ from \cref{eq:params_L} imply that $L/\ell = (\lmax + c_1\lmin) / (\lmin - c_1\lmin) = \kappa/(1-c_1) + c_1 / (1-c_1)$, we have
\begin{equation*}
    \sqrt{\beta} 
    = \frac{\sqrt{L/\ell} - 1}{\sqrt{L/\ell}+1}
    = 1 - \frac{2}{\sqrt{L/\ell}+1}
    = 1 - \frac{2\sqrt{1-c_1}}{\sqrt{\kappa+c_1}+\sqrt{1-c_1}}.
\end{equation*}
Now, set $k^*/\log(k^*) = \sqrt{\kappa}/c_2$ for some $c_2 > 0$.
Then, for
\begin{equation*}
    1-\delta 
    = 1 - \frac{\log(k^*)}{k^*\log(1/\sqrt{\beta})}
    = 1 - \frac{c_2}{\sqrt{\kappa} \log(1/\sqrt{\beta})}
    = 1 - \frac{c_2}{2\sqrt{1-c_1}} + \mathcal{O}(\kappa^{-1}).
\end{equation*}
we have 
\begin{equation*}
    (k^*)^{1/k^*} \sqrt{\beta}
    = (\sqrt{\beta})^{1-\delta}
    = 1 - \frac{2 \sqrt{1 - c_1} - c_2}{\sqrt{\kappa}} + \mathcal{O}(\kappa^{-1}).
\end{equation*}
Therefore, if we take $c_1 = 1-((c+6)/8)^2 $ and $c_2 = (2-c)/4$ we have that, for $\kappa$ sufficiently large,
\begin{equation*}
    (k^*)^{1/k^*} \sqrt{\beta}
    = 1-\frac{(c+2)/2}{\sqrt{\kappa}} + o(1)
    \leq 1-\frac{c}{\sqrt{\kappa}}.
\end{equation*}
Using \cref{thm:HBM_condno} we have that
\begin{equation*}
    \left(\alpha \lmin \condCbd(\alpha,\beta) \right)^2
    \leq \frac{4 (\lmin)^2}{{\gamma(\gamma + \lmax-\lmin)}}
    \leq \frac{\lmin}{\gamma}\frac{4{\lmin}}{{\lmax-\lmin}}
    = \mathcal{O}(\kappa^{-1}).
\end{equation*}
This, with the fact that $\beta = \mathcal{O}(1)$, implies 
\begin{equation*}
    \frac{\| \vec{A} \|_\F^2 \| \vec{A} \|^2  \alpha^2 \condCbd(\alpha,\beta)^2 k^*}{\beta \log(k^*)}
    = \frac{ (d \bar{\kappa} \lmin) (\kappa \lmin) \alpha^2 \condCbd(\alpha,\beta)^2 (4 \sqrt{\kappa})}{\beta}
    = \mathcal{O}(d\bar{\kappa} \sqrt{\kappa})
\end{equation*}
and 
\begin{equation*}
    \left(\frac{\| \vec{A} \|_\F^4 \alpha^2 \condCbd(\alpha,\beta)^2 k^*}{\beta \log(k^*)}\right)^{1/2}
    = \left(\frac{(d\bar{\kappa} \lmin)^2 \alpha^2 \condCbd(\alpha,\beta)^2 (4 \sqrt{\kappa})}{\beta}\right)^{1/2}
    =  \mathcal{O}({d\bar{\kappa} /\sqrt[4]{\kappa}}).
\end{equation*}
Thus, the bound on $B$ becomes $B \geq \mathcal{O}( \eta d\log(d)\bar{\kappa} \sqrt{\kappa} )$.
\end{proof}

\subsection{Inconsistent least squares problems}
\label{sec:inconsistent}

Our results can be extended to inconsistent systems. 
On such systems, the stochastic gradients at the optimal point $\vec{x}^{*}$ need not equal zero, even though $\nabla f(\vec{x}^*) = \vec{0}$.
As a result, stochastic gradient methods will only converge to within a \emph{convergence horizon} of the minimizer $\vec{x}^{*}$.

\begin{remark}
As shown in \citep[Theorem 2.1]{needell2010randomized}, the RK iterates converge to within an expected radius $\sqrt{d \bar{\kappa}} \sigma$ of the least squares solution, where $\sigma = \max_i {|r_i|}/{\|\vec{a}_i\|}$.
The minibatch-RK algorithm from \citep{moorman2020randomized} improves the convergence horizon by roughly a factor of $\sqrt{B}$.
\end{remark}

\begin{theorem} \label{thm:linear_cons_inconsistent}
In the setting of \cref{thm:linear_cons}, define $\vec{r} = \vec{A}\vec{x}^*- \vec{b}$ and $\sigma = \max_i {|r_i|}/{\|\vec{a}_i\|}$. 
Let the batch size $B$ satisfy the conditions in \cref{thm:linear_cons}.
Then, provided $k^*$ is chosen so that $\delta = 2\log(k^*)/(k^*\log(1/\beta)) < 1$, the \ref{stoc_heavyball} iterates satisfy
\begin{equation*}
\mathbb{E} \big[ \left\|\vec{x}_{k} -\vec{x}^{*} \right\| \big]
\leq 
\sqrt{2} \condCbd(\alpha,\beta)  \max\{d,(k^{*})^{k/k^{*}}\}(\sqrt{\beta})^{k} {\left\| \vec{x}_0 -\vec{x}^{*}\right\|} +  R
\end{equation*}
where 
\begin{equation*}
    R \leq  \frac{\alpha \condCbd(\alpha,\beta) (d+1)}{1-(\sqrt{\beta})^{1-\delta}} \left( \left(\frac{2\eta \|\vec{A}\|_\F^2 \log(d+1)\| \vec{r} \|^2}{B}\right)^{1/2}  + \frac{\eta  \|\vec{A} \|_\F^2  \log(d+1)\sigma}{3B} \right).
\end{equation*}
\end{theorem}

\begin{remark}The term in $R$ containing $\|\vec{r}\|$ scales with $1/\sqrt{B}$ whereas the term containing $\sigma$ scales with $1/B$. 
Thus, when $B$ is large, the term in $R$ involving $\sigma$ becomes small relative to the term involving $\|\vec{r}\|$. 
\end{remark}

\begin{remark}
In the large $\kappa$ limit considered in \cref{thm:B_order}, $k^*/\log(k^*) = \mathcal{O}(\sqrt{\kappa})$ and $1/(1-(\sqrt{\beta})^{1-\delta}) = \mathcal{O}(\sqrt{\kappa})$.
Thus, $(k^*+1) / (1-(\sqrt{\beta})^{1-\delta}) = \mathcal{O}(\kappa)$.
\end{remark}

\begin{proof}[Proof of \cref{thm:linear_cons_inconsistent}.]

Since $\vec{b} = \vec{A}\vec{x}^*-\vec{r}$, our stochastic gradients for inconsistent systems satisfy
\begin{equation*}
    \nabla f_{S_k}(\vec{x}_{k}) = 
    \frac{1}{|S_k|} \sum_{i \in S_k} \left( \frac{1}{p_i} \vec{a}_i \vec{a}_i^\T ( \vec{x}_{k} -\vec{x}^{*})
    + \frac{r_i}{p_i} \vec{a}_i \right)
    = \vec{M}_{S_k} (\vec{x}_{k} -\vec{x}^{*}) + \vec{r}_{S_k}
\end{equation*}
where $\vec{M}_{S_k}$ is as in the proof of \cref{thm:linear_cons}, and define
\begin{equation*}
    \vec{r}_{S_k} = \frac{1}{|S_k|} \sum_{i \in S_k} \frac{r_i}{p_i} \vec{a}_i.
\end{equation*}
We therefore find the iterates satisfy the update formula
\begin{equation*}
\begin{bmatrix}\vec{x}_{k+1} - \vec{x}^{*} \\ \vec{x}_{k} - \vec{x}^{*} \end{bmatrix} 
=  \underbrace{ \begin{bmatrix} (1+ \beta ) \vec{I} - \alpha \vec{M}_{S_k} & - \beta \vec{I} \\
\vec{I} & \vec{0} \end{bmatrix}}_{\vec{Y}_{S_k} = \vec{Y}_{S_k}(\alpha, \beta) } 
\begin{bmatrix} \vec{x}_{k} - \vec{x}^{*} \\\vec{x}_{k-1} - \vec{x}^{*} \end{bmatrix} + 
\alpha \begin{bmatrix} \vec{r}_{S_k} \\ \vec{0} \end{bmatrix}.
\end{equation*}
Thus, after $k$ iterations, the error satisfies 
\begin{equation*}
\begin{bmatrix} \vec{x}_{k+1}-\vec{x}^{*} \\ \vec{x}_k -\vec{x}^{*} \end{bmatrix}
\leq \bigg( \prod_{i=1}^{k} \vec{Y}_{S_i} \bigg)
\begin{bmatrix} \vec{x}_{1} -\vec{x}^{*} \\ \vec{x}_0 -\vec{x}^{*} \end{bmatrix} 
+ \alpha \sum_{j=1}^{k} \bigg( \prod_{i=j+1}^{k} \vec{Y}_{S_i} \bigg) \begin{bmatrix} \vec{r}_{S_j} \\ \vec{0} \end{bmatrix}.
\end{equation*}
The first term is identical to the case $\vec{r} = 0$, so we have
\begin{align*}
\frac{R}{\alpha} \leq \E \bigg[ \bigg\| \sum_{j=1}^{k} \bigg( \prod_{i=j+1}^{k} \vec{Y}_{S_i} \bigg) \begin{bmatrix} \vec{r}_{S_j} \\ 0 \end{bmatrix} \bigg\| \bigg]
&\leq  \sum_{j=1}^{k} \E \bigg[ \bigg\| \bigg( \prod_{i=j+1}^{k} \vec{Y}_{S_i} \bigg) \bigg\| 
\bigg] \E \bigg[ \bigg\| \begin{bmatrix} \vec{r}_{S_{j}} \\ \vec{0} \end{bmatrix} \bigg\| \bigg].
\end{align*}
Here we have used the triangle inequality and definition of operator norm followed by the linearity of expectation and independence of the minibatch draws across iterations.

As in the proof of \cref{thm:linear_cons},
\begin{equation*}
    \E \big[ \| \vec{Y}_{S_{k}} \vec{Y}_{S_{k-1}}\cdots \vec{Y}_{S_{j+1}} \| \big]
    \leq \condCbd(\alpha,\beta) \max\{d,(k^*)^{(k-j)/k^*}\} (\sqrt{\beta})^{k-j}.
\end{equation*}
Therefore, assuming all minibatch draws $\{S_j\}$ are identically distributed, we have
\begin{align*}
    \frac{R}{\alpha}
    &\leq  \bigg( \sum_{j=1}^{k} \condCbd(\alpha,\beta) \max\{d,(k^*)^{(k-j)/k^*}\} (\sqrt{\beta})^{k-j} \bigg) \E [\| \vec{r}_{S_{1}} \|]
    \\&\leq  \condCbd(\alpha,\beta) \bigg( \sum_{j=0}^{k-1} \max\{d,(k^*)^{j/k^*}\} (\sqrt{\beta})^{j}  \bigg) \E [\| \vec{r}_{S_{1}} \|]
    \\&\leq   \condCbd(\alpha,\beta) \bigg( \sum_{j=0}^{k-1} d (\sqrt{\beta})^{j}  + \sum_{j=0}^{k-1} (k^*)^{j/k^*}  (\sqrt{\beta})^{j}  \bigg) \E [\| \vec{r}_{S_{1}} \|]
    \\&= \condCbd(\alpha,\beta) \bigg( d\frac{1-(\sqrt{\beta})^{k}}{1-\sqrt{\beta}} + \frac{1-(k^*)^{k/k^*}(\sqrt{\beta})^{k}}{1-(k^*)^{1/k^*}\sqrt{\beta}} \bigg) \E [\| \vec{r}_{S_{1}} \|]. 
\end{align*}
Now, note that, by assumption\footnote{Without this assumption we still have a (messy) bound for $R$.}, $(k^*)^{1/k^*} \sqrt{\beta} = (\sqrt{\beta})^{1-\delta} < 1$. Thus,
\begin{equation*}
    d\frac{1-(\sqrt{\beta})^{k^*}}{1-\sqrt{\beta}} + \frac{1-((k^*)^{1/k^*}\sqrt{\beta})^{k}}{1-(k^*)^{1/k^*}\sqrt{\beta}}
    \leq \frac{d+1}{1-(\sqrt{\beta})^{1-\delta}}.
\end{equation*}

It remains to bound $\E[ \| \vec{r}_{S_1}\| ]$, and to do so we again turn to the matrix Bernstein inequality. 
Define the sum
\begin{equation*}
\vec{Z} = \sum_{i \in  S_k} \frac{r_{i}}{p_i} \vec{a}_i
= \vec{X}_1 + \cdots + \vec{X}_B
\end{equation*}
and note that, with $\sigma = \max_i {|r_i|}/{\|\vec{a}_i\|}$, and the assumption \cref{eqn:sampling_probs} on the sampling probabilities,
\begin{equation*}
    \|\vec{X}_i\| 
    \leq \frac{\|\vec{a}_i\| |r_i|}{\|\vec{a}_i\|^2/(\eta \|\vec{A}\|_\F^2)}
    \leq \eta \sigma \|\vec{A} \|_\F^2  .
\end{equation*}
By direct computation we also observe that
\begin{equation*}
   \bigg\| \sum_{i\in S_k} \E \big[ \vec{X}_i^\T \vec{X}_i \big] \bigg\|
    \leq \bigg\| \sum_{i\in S_k} \E \bigg[ \frac{\eta \| \vec{A} \|_\F^2r_i^2}{\|\vec{a}_i\|^2p_i} \vec{a}_i^\T \vec{a}_i \bigg] \bigg\|
    =  \eta B \| \vec{A} \|_\F^2 \sum_{i=1}^{n} r_i^2
    =  \eta B \| \vec{A} \|_\F^2 \| \vec{r} \|^2
\end{equation*}
and, since $\big\| \sum_{i\in S_k} r_i^2 \vec{a}_i \vec{a}_i^\T/\|\vec{a}_i\|^2 \big\| \leq \sum_{i\in S_k} r_i^2$, that
\begin{equation*}
    \bigg\| \sum_{i\in S_k} \E \big[ \vec{X}_i \vec{X}_i^\T \big] \bigg\|
    \leq \bigg\| \sum_{i\in S_k} \E \bigg[ \frac{\eta \| \vec{A} \|_\F^2 r_i^2}{\|\vec{a}_i\|^2p_i}  \vec{a}_i \vec{a}_i^\T \bigg] \bigg\|
    \leq \eta B \| \vec{A} \|_\F^2 \sum_{i=1}^{n} r_i^2
    = \eta B \| \vec{A} \|_\F^2 \| \vec{r} \|^2.
\end{equation*}
Therefore, applying \cref{thm:matrixmoment} we obtain the bound
\begin{equation*}
    \E \big[ \| \vec{r}_{S_k}\| \big] 
    = \frac{1}{B}\E \big[ \|\vec{Z}\| \big]
    \leq 
    \sqrt{\frac{2\eta\|\vec{r}\|^2\|\vec{A}\|_\F^2 \log(d+1)}{B}} + \frac{\eta \sigma \|\vec{A} \|_\F^2  \log(d+1)}{3B}.
\end{equation*}
Combining everything gives the desired result.
\end{proof}

\section{Numerical Results}

We conduct numerical experiments on quadratic objectives \cref{eqn:lls} to illustrate the performance of \ref{stoc_heavyball}.
Throughout this section, we use the value
\begin{equation*}
    B^* = \frac{16 \mathrm{e} \| \vec{A} \|_\F^2 \| \vec{A}\|^2  \log(2d) \alpha^2 }{ \beta \log(1/\beta)}
\end{equation*}
as a \emph{heuristic} for the batch size needed to observe linear convergence at a rate similar to that of \ref{heavyball}.
This value is obtained from \cref{thm:linear_cons} by making several simplifying assumptions. 
Specifically we drop the dependence on $\condCbd(\alpha,\beta) $ which results from a change of basis followed by a return to the original basis (which we believe is likely an artifact of our analysis approach) and replace $k^*/\log(k^*)$ by $1/\log(1/\beta) = \mathcal{O}(\sqrt{\kappa})$.

In all experiments we use $n=10^6$, $d=10^2$ and set $\gamma = \lmin/10^3$.
Each experiment is repeated 100 times and the median and 5th to 95th percentile range for each algorithm/parameter choice are plotted.

\subsection{Row-norm and uniform sampling}

\begin{figure}[htb]
    \centering
    \includegraphics[width=\textwidth]{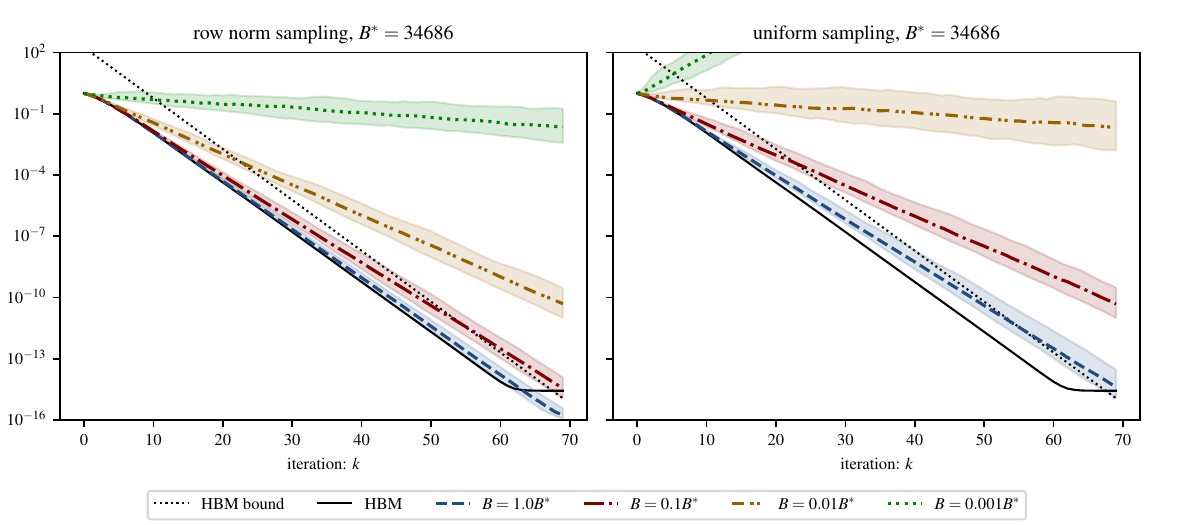}
    \caption{Median and 5th to 95th percentile error norm $\|\vec{x}_k - \vec{x}^*\|$ of \ref{stoc_heavyball} for row norm sampling and uniform sampling at varying values of batch size $B$.
    For reference, we also show the convergence of \ref{heavyball} and the \ref{heavyball} bound \cref{eqn:heavyball_bd}.
    }
    \label{fig:sampling_probs}
\end{figure}

In this example, we study the dependence of \ref{stoc_heavyball} on the sampling probabilities. 
Our bounds are sharpest when $p_j\propto \|\vec{a}_j\|^2$, but in practice it is common to use uniform sampling $p_j = 1/n$ to avoid the need for computing row-norms which requires accessing the entire data matrix.
We take $\vec{A} = \vec{D} \vec{G}$, where $\vec{G}$ is an $n\times d$ matrix whose entries are independently 1 with probability $1/10$ or 0 with probability $9/10$ and $\vec{D}$ is a $n\times n$ diagonal matrix whose diagonal entries are 1 with probability $9/10$ and $10$ with probability $1/10$.
Thus, uniform sampling probabilities $p_j=1/n$ satisfy \cref{eqn:sampling_probs} provided $\eta \geq n \max_j \| \vec{a}_j \|^2/\|\vec{A}\|_\F^2 \approx 23$.
We use a planted solution $\vec{x}^*$ with iid standard normal entries.

In \cref{fig:sampling_probs} we show the convergence of \ref{stoc_heavyball} with row norm sampling and uniform sampling at several values of $B$.
As expected, row norm sampling works better than uniform sampling for a fixed value of $B$.
However, since the norms of rows are not varying too significantly, the convergence rates are still comparable.
See \citep{needell2014stochastic} for a further discussion on sampling probabilities in the context of RK and SGD.

\subsection{Sensitivity to batch size}

The fact \ref{stoc_heavyball} exhibits accelerated convergence is an artifact of batching, and we expect different behavior at different batch sizes.
In fact, we have already observed this phenomenon on the previous example.
In \cref{thm:linear_cons}, we provide an upper bound on the required batch size depending on spectral properties of $\ATA$ such as $\bar{\kappa} = \lave / \lmin$ and $\kappa = \lmax / \lmin$.
To study whether the stated dependence on such quantities is reasonable, we construct a series of synthetic problems with prescribed spectra.

\begin{figure}[htb]
    \centering
    \includegraphics[width=\textwidth]{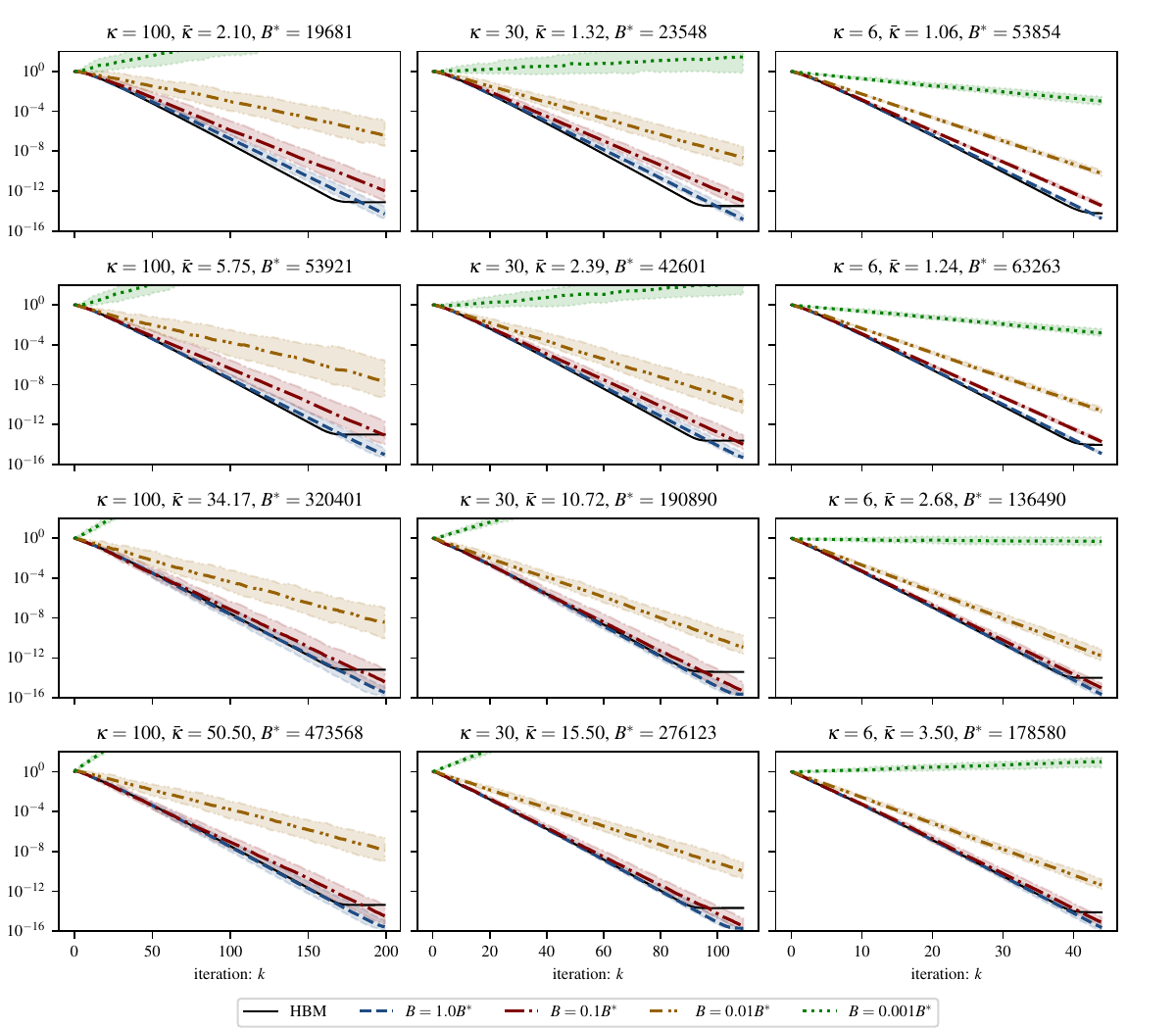}
    \caption{Median and 5th to 95th percentile error norm error norm $\|\vec{x}_k - \vec{x}^*\|$ of \ref{stoc_heavyball} for varying values of batch size $B$ on problems with a range of $\kappa$ and $\bar{\kappa}$.
    For reference, we also show the convergence of \ref{heavyball}.}
    \label{fig:B_dependence}
\end{figure}

\begin{figure}[htb]
    \centering
    \includegraphics[width=\textwidth]{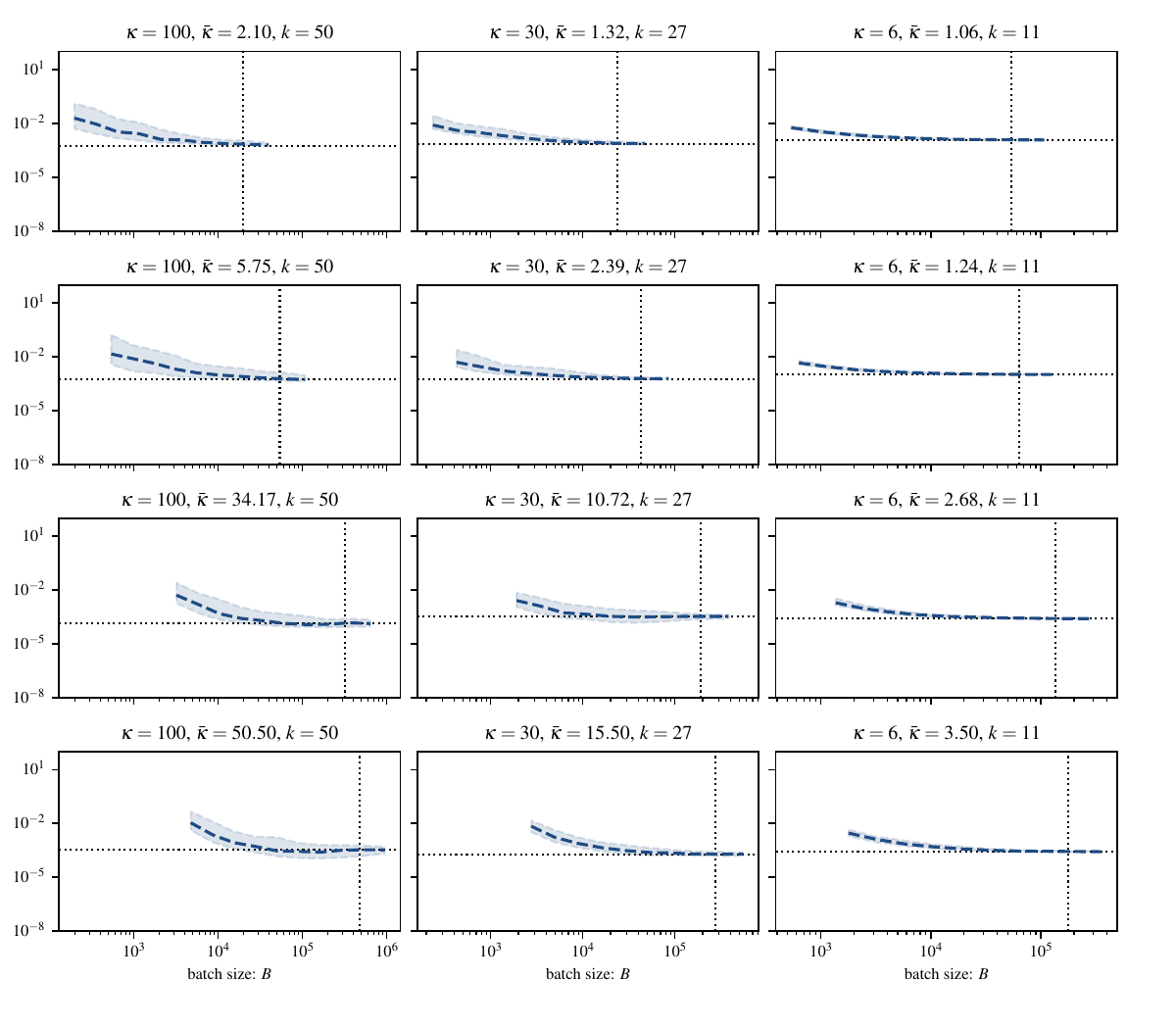}
    \caption{Median and 5th to 95th percentile error norm $\|\vec{x}_{k} - \vec{x}^*\|$ of \ref{stoc_heavyball} for varying values of batch size $B$ on problems with a range of $\kappa$ and $\bar{\kappa}$. The horizontal dotted lines indicate the accuracy of \cref{heavyball}, and the vertical dotted lines indicate $B = B^*$.}
    \label{fig:B_dependence_kfix}
\end{figure}

We construct problems $\vec{A} = \vec{U}\bm{\Sigma}\vec{V}^\T$ by choosing the singular vectors $\vec{U}$ and $\vec{V}$ uniformly at random, and selecting singular values $\{\sigma_i\}$ with exponential or algebraic decay.
For exponential decay we use the squared singular values\footnote{This spectrum is often referred to as the `model problem' in numerical analysis and is commonly used to study the convergence of iterative methods for solving linear systems of equations \citep{strakos_91,strakos_greenbaum_92}.}
\begin{equation}
\label{eqn:exponential_spec}
    \sigma_j^2 = 1 + \left( \frac{j-1}{d-1} \right)  (\kappa -1)  \rho^{d-j}
    ,\qquad j=1,2,\ldots, d,
\end{equation}
and for algebraic decay we use the squared singular values
\begin{equation}
\label{eqn:algebraic_spec}
    \sigma_j^2 = 1 + \left(\frac{j-1}{d-1}\right)^\rho (\kappa-1)
    ,\qquad j=1,2,\ldots, d.
\end{equation}
In both cases, the condition number of the $\ATA$ is $\kappa$ and $\rho$ determines how fast the singular values of $\vec{A}$ decay.
We again use a planted solution $\vec{x}^*$ with iid standard normal entries.

In \cref{fig:B_dependence,fig:B_dependence_kfix} we report the results of our experiments. 
Here we consider consistent equations with condition numbers $\kappa = 10,30,100$.
For each value of $\kappa$, we generate two problems according to \cref{eqn:exponential_spec} with $\rho = 0.1$ and $\rho = 0.8$ and two problems according to \cref{eqn:algebraic_spec} with $\rho=2$ and $\rho=1$.
In \cref{fig:B_dependence} we run \ref{stoc_heavyball} (row norm sampling) with $B=cB^*$ for $c=10^{-3},10^{-2},10^{-1},10^0$ and show the convergence as a function of the iterations $k$.
In \cref{fig:B_dependence_kfix} we run \cref{stoc_heavyball} for a fixed number of iterations, and show the convergence as a function of the batch size $B$.
This empirically illustrates that when the batch size is near $B^*$, the rate of convergence is nearly that of \cref{heavyball}.

\subsection{Inconsistent systems}

For inconsistent systems, stochastic gradients at the optimum will not be zero and convergence is only possible to within the so-called \emph{convergence horizon}.
In this experiment we sample $\vec{A}$ as in the previous experiment using \cref{eqn:exponential_spec} with $\kappa = 50$ and $\rho = 0.5$. 
We take $\vec{b} = \vec{A}\vec{x} + \vec{\epsilon}$, where $\vec{x}$ is has iid standard normal entries and $\vec{\epsilon}$ is drawn uniformly from the hypersphere of radius $10^{-5}$.
Thus, the minimum residual norm $\|\vec{b} - \vec{A} \vec{x}^* \|$ is around $10^{-5}$. 
We use several different values of $B$.

\begin{figure}[htb]
    \centering
    \includegraphics[width=\textwidth]{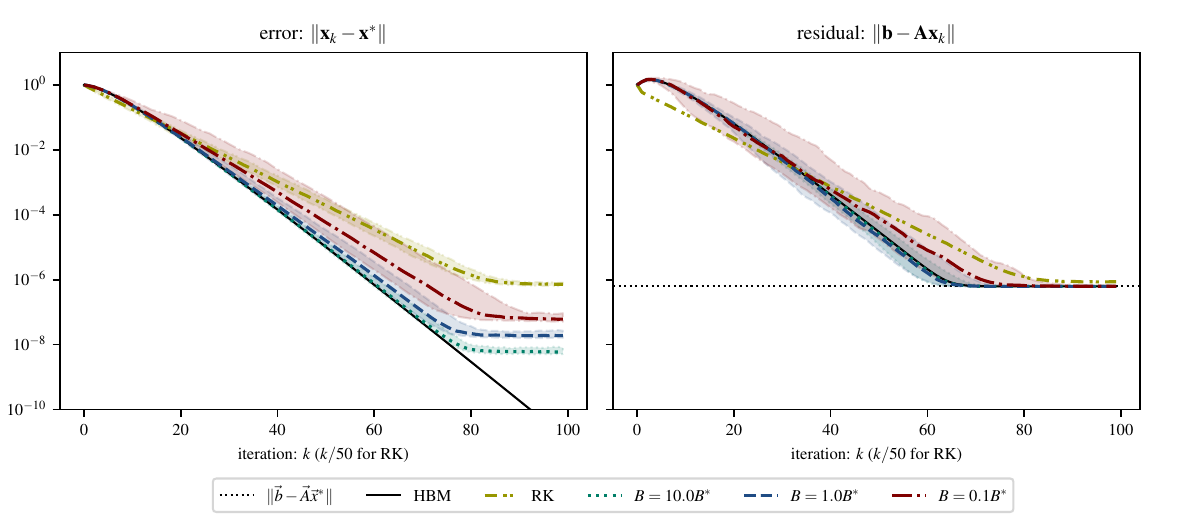}
    \caption{Median and 5th to 95th percentile error norm and residual norm of \ref{stoc_heavyball} for varying values of batch size $B$ on an inconsistent problem.
    For reference, we also show the convergence of \ref{heavyball} as well as optimal residual norm.}
    \label{fig:inconsistent}
\end{figure}

The results of the experiment are shown in \cref{fig:inconsistent}.
For larger batch sizes, the convergence horizon of \ref{stoc_heavyball} gets smaller. 
In particular, when the batch size is increased by a factor of 100, the convergence horizon is decreased by about a factor of 10.
This aligns with the intuition that the convergence horizon should depend on $\sqrt{B}$.
For reference, we also show the convergence for standard RK.
Note RK requires more iterations to converge, although each iteration involves significantly less computation.%
\footnote{While \ref{stoc_heavyball} uses significantly more floating point operations than RK, the runtime to fixed accuracy (for all batch sizes tested) was actually significantly lower than RK due to vectorized operations.
Since runtime is highly system dependent, we do not provide actual timings.}
We also note that, while the error in the iterates stagnates at different points, the value of the objective function is quite similar in all cases, nearly matching the residual norm of the true least squares solution.

\subsection{Computational Tomography}
One of the most prevalent applications of Kaczmarz-like methods is in tomographic image reconstruction, notably in medical imaging.
In X-ray tomography (e.g. CT scans), X-rays are passed through an object and the intensity of resulting X-ray beam is measured.
This process is repeated for numerous known orientations of the X-ray beams relative to the object of interest. 
With a sufficient number of measurements, it is possible to reconstruct a ``slice'' of the interior of the object of interest.
In theory, this reconstruction involves solving a large, sparse consistent linear system.

In this example we consider the performance of \ref{stoc_heavyball} on a tomography problem corresponding to a parallel beam geometry scanner with $128$ sensor pixels. Measurements are taken for $360$ degrees of rotation at half-degree increments, totaling 720 measurements. The goal is to reconstruct a $64\times 64$ pixel image. This results in a measurement matrix of dimensions $(720\cdot 128) \times(64\cdot 64) = 92160\times 4096$ which we construct using the ASTRA toolbox \citep{vanAarle2015}.

We employ a planted solution of a walnut, which we aim to recover from the resulting measurement data.  Uniform sampling is employed due to the roughly similar norms of all rows. 
The step-size $\alpha$ and momentum parameter $\beta$ are chosen so that \ref{heavyball} converges at a reasonable rate and so that $B^* = 407$ is not too large relative to $n$.
In \cref{fig:tomography} we report the convergence of \ref{heavyball} and \ref{stoc_heavyball}, along with the resulting images recovered by the algorithms after $k=500$ iterations.

\begin{figure}[htb]
    \centering
    \includegraphics[width=\textwidth]{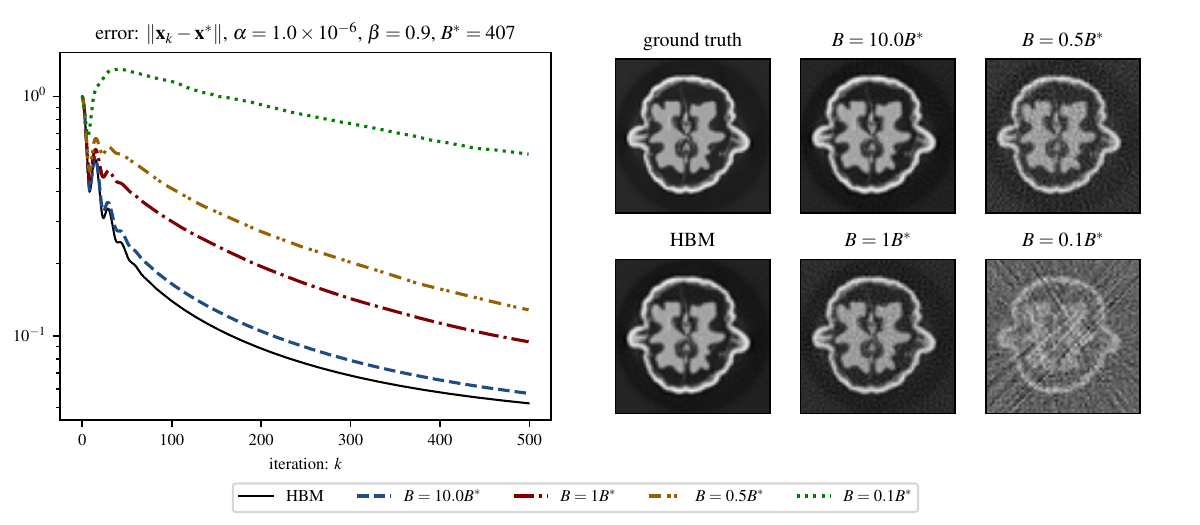}
    \caption{Error $\|\vec{x}_k - \vec{x}^*\|$ of \ref{stoc_heavyball} on a problem from computational tomography at varying batch sizes, and the resulting images of the interior of a walnut recovered after $k=500$ iterations.}
    \label{fig:tomography}
\end{figure}

\section{Conclusion}
We provided a first analysis of accelerated convergence of minibatch heavy ball momentum method for quadratics using standard choices of the momentum step-sizes.  Our proof method involves a refined quantitative analysis of the convergence proof for (deterministic) heavy ball momentum, combined with matrix concentration results for sums and products of independent matrices.  Our proof technique is general, and also can be used to verify the accelerated convergence of a minibatch version of Nesterov's acceleration for quadratics, using the constant step-size parameters suggested in Nesterov's original paper.  An interesting direction for future work is to combine the simple minibatch momentum algorithm with an adaptive gradient update such as AdaGrad \citep{duchi2011adaptive,ms2010, ward2019adagrad,defossez2020convergence}, to potentially learn the momentum parameters $\alpha_k$ and $\beta_k$ adaptively.  Such an analysis would also shed light on the convergence of ADAM \citep{kingma2014adam}, an extension of momentum which combines adaptive gradients and momentum in a careful way to achieve state-of-the-art performance across various large-scale optimization problems.    

\section*{Acknowledgments}
We thank Qijia Jiang and Stephen Wright for helpful comments during the preparation of this manuscript.

\section*{Funding}
R. Bollapragada was supported by NSF DMS 2324643. T. Chen was supported by NSF DGE 1762114.  R. Ward was partially supported by AFOSR MURI FA9550-19-1-0005, NSF DMS 1952735, NSF HDR 1934932, and NSF CCF 2019844.

\appendix

\section{Analysis of Nesterov's acceleration for quadratics}
\label{sec:NAG}

Another common approach to accelerating the convergence of gradient descent is Nesterov's accelerated gradient descent (NAG).
\ref{eqn:NAG} uses iterates
\begin{equation*}
    \vec{x}_{k+1} = \vec{y}_{k}- \alpha_k \nabla f(\vec{y}_k)
    ,\qquad
    \vec{y}_{k+1} = \vec{x}_{k+1} + \beta_k(\vec{x}_{k+1} - \vec{x}_k)
\end{equation*}
or equivalently,
\begin{equation}
    \label{eqn:NAG}
    \tag{NAG}
    \vec{x}_{k+1} = \vec{x}_k + \beta(\vec{x}_{k} - \vec{x}_{k-1})- \alpha  \big(\ATA (\vec{x}_k + \beta(\vec{x}_{k} - \vec{x}_{k-1})) - \vec{A}^\T \vec{b} \big)
\end{equation}
Therefore, a computation analogous to the above computation for \ref{heavyball} shows that the \ref{eqn:NAG} iterates satisfy the transition relation
\begin{equation}
\begin{bmatrix}\vec{x}_{k+1} - \vec{x}^{*} \\ \vec{x}_{k} - \vec{x}^{*} \end{bmatrix} 
=  \underbrace{ \begin{bmatrix} (1+ \beta ) (\vec{I} - \alpha \ATA) & - \beta (\vec{I} - \alpha \ATA) \\
\vec{I} & \vec{0} \end{bmatrix}}_{\vec{T} = \vec{T}(\alpha, \beta) } \begin{bmatrix} \vec{x}_{k} - \vec{x}^{*} \\\vec{x}_{k-1} - \vec{x}^{*} \end{bmatrix}.
\end{equation}
Again, $\vec{T}$ is unitarily similar to a block diagonal matrix whose blocks are
\begin{equation*}
    \vec{T}_j = 
    \begin{bmatrix}
    (1+\beta)(1-\alpha \lambda_j) & -\beta(1-\alpha \lambda_j) \\ 1 & 0
    \end{bmatrix},
\end{equation*}
and it is easy to see the eigenvalues of $\vec{T}_j$ are
\begin{equation*}
    z_j^\pm := \frac{1}{2} \left( (1+ \beta) (1- \alpha \lambda_j)  \pm \sqrt{(1+ \beta)^2 (1- \alpha \lambda_j)^2-4\beta(1- \alpha \lambda_j)} \right).
\end{equation*}
Rather than aiming to optimize the parameters $\alpha$ and $\beta$ as we did for \ref{heavyball}, we will simply use the standard choices of parameters suggested in \citep{Nesterov_2013}:
\begin{align*}
    \alpha = \frac{1}{L}
    \qquad\text{and}\qquad
    \beta = \frac{\sqrt{L/\ell}-1}{\sqrt{L/\ell}+1}.
\end{align*}

By direct computation we find
\begin{equation*}
    \frac{4\beta}{(1+ \beta)^2} = 1 - \frac{1}{L/\ell}
\end{equation*}
which implies 
\begin{equation*}
    (1+ \beta)^2 (1- \alpha \lambda_j)^2 \leq 4\beta(1- \alpha \lambda_j)
\end{equation*}
and therefore that
\begin{align*}
    |z_j^\pm| 
    &= \sqrt{\beta(1-\alpha \lambda_j)}
    \leq \sqrt{\beta(1-\alpha \ell)}
    = 1 - \frac{1}{\sqrt{L/\ell}}.
\end{align*}
Thus, the \ref{eqn:NAG} iterates satisfy the convergence guarantee
\begin{equation*}
    \frac{\left\|\vec{x}_{k+1} - \vec{x}^{*} \right\|}{\left\|  \vec{x}_0 - \vec{x}^{*} \right\|}
    \leq \sqrt{2} \condCbd(\alpha,\beta) \left( 1 - \frac{1}{\sqrt{L/\ell}} \right)^{k},
\end{equation*}
where $\condCbd(\alpha,\beta) $ is the eigenvector condition number of the transition matrix $\vec{T}$ (note that this value is different from the value for \ref{heavyball} bounded in \cref{thm:HBM_condno}).

We now provide a bound, analogous to \cref{thm:linear_cons}, for minibatch-NAG, the algorithm obtained by using the stochastic gradients \cref{eqn:mb_gradient} in \cref{eqn:NAG}.
\begin{theorem} \label{thm:linear_cons_NAG}
Set $\ell = \lmin$, $L = \lmax$ and define $\alpha = 1/L$ and $\beta = (\sqrt{L/\ell}-1)/(\sqrt{L/\ell}+1)$.
For any $k^*>1$  choose
\begin{equation*}
    B 
    \geq 16 \mathrm{e} \eta \log(2d) \max\left\{ \frac{5 \| \vec{A} \|_\F^2 \| \vec{A} \|^2\alpha^2 K^2 k^*}{ \beta \log(k^*)}, \left( \frac{10\| \vec{A} \|_\F^4 \alpha^2 K^2 k^*}{\beta \log(k^*)} \right)^{1/2} \right\}.
\end{equation*}
Then, for all $k>0$, assuming that the minimizer $\vec{x}^*$ satisfies $\vec{A}\vec{x}^* = \vec{b}$, the \ref{eqn:NAG} iterates satisfy
\begin{equation*}
\E \big[ {\left\|\vec{x}_{k} -\vec{x}^{*} \right\|} \big]
\leq \sqrt{2} \condCbd(\alpha,\beta) \max\{d,(k^*)^{k/k^*}\} \left(1- \frac{1}{\sqrt{L/\ell} + 1}\right)^{k} {\left\| \vec{x}_0 -\vec{x}^{*}\right\|}.
\end{equation*}
\end{theorem}

The proof of \cref{thm:linear_cons_NAG} is almost the same as the proof identical to the proof of \cref{thm:linear_cons}, so we skip repeated parts.
\begin{proof}
For NAG we have that
\begin{align*}
\vec{Y}_{S_i} - \E[\vec{Y}_{S_i}]
&=\sum_{j\in S_i}\frac{1}{B}
\begin{bmatrix} (1+ \beta ) \alpha ( - p_j^{-1} \vec{a}_j\vec{a}_j^\T + \ATA) & - \beta \alpha ( - p_j^{-1} \vec{a}_j\vec{a}_j^\T + \ATA) \\
\vec{0} & \vec{0} \end{bmatrix}
\\
&=\sum_{j\in S_i}\frac{\alpha }{B}
\begin{bmatrix}
-p_j^{-1} \vec{a}_j\vec{a}_j^\T +\ATA  \\
& \vec{0}
\end{bmatrix}
\begin{bmatrix} (1+ \beta ) \vec{I}  & - \beta \vec{I} \\
\vec{I} & \vec{0} 
\end{bmatrix}
\\
&=
\alpha
\left(\sum_{j\in S_i}
\begin{bmatrix}
\vec{W}_j  \\
& \vec{0}
\end{bmatrix}
\right)
\begin{bmatrix} (1+ \beta ) \vec{I}  & - \beta \vec{I} \\
\vec{0}  & \vec{0} \end{bmatrix}
\end{align*}
Note that, since $\sqrt{\beta}\leq 1$,
\begin{equation*}
    \left\| \begin{bmatrix} (1+ \beta ) \vec{I}  & - \beta \vec{I} \\
\vec{0}  & \vec{0} \end{bmatrix} \right\|
= \left\| \begin{bmatrix} (1+ \beta )  & - \beta  \\
0&0  \end{bmatrix} \right\|
= \sqrt{1+2\beta + 2\beta^2}
\leq \sqrt{5}.
\end{equation*}
Thus, using the submultiplicativity of the operator norm, analogous to \cref{eqn:vY_bound}, we have that
\begin{equation*}
    \sqrt{\E \big[ \| \vec{X}_{S_i} - \E [\vec{X}_{S_i} ] \|^2 \big] }
    \leq \sqrt{\condCbd(\alpha,\beta) \E\big[ \| \vec{Y}_{S_i} - \E [\vec{Y}_{S_i}] \|^2 \big]}
    \leq \sqrt{5}\alpha \condCbd(\alpha,\beta) \sqrt{\E \big[ \| \vec{W} \|^2 \big]}.
\end{equation*}
Again, using \cref{thm:sqnormbd,eqn:vY_bound}, we see that $\sqrt{\E\|\vec{X}_{S_i} - \E \vec{X}_{S_i}\|^2} \leq \delta$ provided that the batch size $B$ satisfies
\begin{equation*}
    B \geq {8 \mathrm{e}\eta \log(2d)} \max\big\{ 5 \| \vec{A} \|_\F^2  \|\vec{A}\|^2 \alpha^2 \condCbd(\alpha,\beta) ^2 \delta^{-2} , ( 20 | \vec{A} \|_\F^4 \alpha^2 \condCbd(\alpha,\beta) ^2 \delta^{-2} )^{1/2} \big\}.
\end{equation*}
Using the same choice of $\delta^2 = \beta \log(k^*)/(2k^*)$ we get the desired bound. 
\end{proof}

\bibliographystyle{apalike}
\bibliography{references}

\end{document}